%% file: main.tex
\documentclass[table]{gtech}
\PassOptionsToPackage{table, usenames, dvipsnames}{xcolor}
\usepackage{bigdelim}
\usepackage{longtable}
\usepackage{tabularray}
\usepackage{float}
\usepackage{datatool}
\usepackage[justification=centering]{caption}

\RequirePackage{tgpagella} 
\RequirePackage{mathpazo}  
\usepackage{times}
\usepackage{latexsym}
\usepackage[T1]{fontenc}
\usepackage[utf8]{inputenc}
\usepackage{microtype}

\newcommand{\ignore}[1]{}

\renewcommand{\title}[1]{\newcommand{\titlelist}{{\huge\selectfont #1}}}

\usepackage{arydshln}
\definecolor{CQColor}{rgb}{0.0,0.0,1.0} 

\usepackage{pifont}
\usepackage{tabulary}
\usepackage{fontawesome5}
\usepackage{bbding}
\usepackage{multicol}

\newlength\savewidth

\usepackage{hyperref}
\usepackage{url}

\usepackage{makecell}  
\usepackage{graphicx}
\usepackage{multirow}
\usepackage{color}
\usepackage{array}
\usepackage{algorithm}
\usepackage{algpseudocode}
\usepackage{wrapfig}
\usepackage{amsmath}
\usepackage{makecell}
\usepackage{colortbl}
\usepackage{xcolor}
\usepackage{caption}
\usepackage{subcaption}  
\definecolor{darkred}{RGB}{162, 0, 0}

\definecolor{darkblue}{RGB}{4, 6, 173}

\definecolor{darkgreen}{RGB}{61, 134, 73}

\algnewcommand{\algorithmicinit}{\textbf{Initialize:}}
\algnewcommand{\Init}{\algorithmicinit}
\usepackage{amsmath,amsfonts,bm}
\usepackage{amsthm}
\usepackage{xspace}
\usepackage[most]{tcolorbox}
\usepackage{enumitem}
\usepackage{amssymb}
\usepackage{mathtools}
\usepackage{physics} 
\usepackage{thmtools}

\usepackage{amsfonts,bm}
\usepackage{dsfont}  
\usepackage{diagbox} 

\newtheorem{theorem}{Theorem}[section]
\newtheorem{lemma}[theorem]{Lemma}

\newtheorem{definition}[theorem]{Definition}

\usepackage{tabularx}
\tcbuselibrary{skins}

\usepackage{titletoc}   

\usepackage{booktabs}

\tcbuselibrary{breakable}
\definecolor{my_green}{RGB}{51,102,0}
\definecolor{my_purple}{RGB}{160, 43, 147}
\definecolor{my_blue}{RGB}{15, 158, 213}

\NewDocumentCommand{\ganqu}
{ mO{} }{\textcolor{blue}{\textsuperscript{\textit{ganqu}}\textsf{\textbf{\small[#1]}}}}
\NewDocumentCommand{\yafu}
{ mO{} }{\textcolor{cyan}{\textsuperscript{\textit{yafu}}\textsf{\textbf{\small[#1]}}}}

\NewDocumentCommand{\jianhao}
{ mO{} }{\textcolor{red}{\textsuperscript{\textit{jianhao}}\textsf{\textbf{\small[#1]}}}}

\definecolor{deltaBg}{RGB}{220,230,255} 

\newtcolorbox{remark}[1][]{enhanced,
  breakable,
  colback=violet!6,           
  colframe=violet!50!black,   
  coltitle=black,             
  colbacktitle=violet!18,     
  fonttitle=\bfseries,
  title=Remark,
  #1}

\definecolor{uclablue}{rgb}{0.15, 0.45, 0.68}
\hypersetup{
    breaklinks,
    citecolor=uclablue,
    colorlinks=true,
}

\definecolor{lightgreen}{RGB}{0,150,0}
\definecolor{myred}{RGB}{200,0,0} 

\newcommand{\cimp}[2]{%
  \ensuremath{%
    #1%
    \mathrlap{_{\scriptscriptstyle\textcolor{lightgreen}{#2}}}%
  }%
}

\newcommand{\cdec}[2]{%
  \ensuremath{%
    #1%
    \mathrlap{_{\scriptscriptstyle\textcolor{myred}{#2}}}%
  }%
}

\title{\textbf{Can LLMs Learn to Reason Robustly under Noisy Supervision?}}
\author[1,2]{Shenzhi Yang}
\author[1,2]{Guangcheng Zhu}
\author[2]{Bowen Song}
\author[3]{Sharon Li}
\author[1\ddag]{Haobo Wang}
\author[2]{Xing Zheng}
\author[2]{Yingfan Ma}
\author[2]{Zhongqi Chen}
\author[2]{Weiqiang Wang}
\author[1]{Gang Chen}

\affiliation[1]{Zhejiang University $^2$Ant Group $^3$University of Wisconsin-Madison \\[0.5em]}

\contribution[\ddag]{Corresponding authors.}

\abstract{\fontsize{11pt}{12pt} \textit{
Reinforcement Learning with Verifiable Rewards (RLVR) effectively trains reasoning models that rely on abundant perfect labels, but its vulnerability to unavoidable noisy labels from expert scarcity remains critically underexplored.
In this work, we take the first step toward a systematic analysis of noisy label mechanisms in RLVR.
In contrast to supervised classification, most RLVR algorithms incorporate a \textbf{rollout-based} condition: a label's influence on training is contingent on whether the current policy can generate rollouts that realize it, a property that naturally extends to noisy labels.
Based on this observation, we distinguish two types of noise: \textbf{inactive noisy labels}, which reduce data efficiency, and \textbf{active noisy labels}, which are reinforced and risk skewing the model toward incorrect distributions.
From experiments on training with noisy samples, we identify a \textbf{Early Correctness Coherence} phenomenon: \textit{although noisy samples begin to lag behind in later stages, accuracy on both clean and noisy samples increases similarly in early training.}
Motivated by this dynamic, we propose \textbf{Online Label Refinement (OLR)}, which progressively corrects potentially noisy labels with majority-voted answers when two conditions hold: a \textbf{{positive slope}} in the majority answer’s rollout pass rate and stable \textbf{{historical consistency}} across updates, enabling gradual self-correction as the policy improves.
We evaluate OLR on six in-distribution mathematical reasoning benchmarks (AIME24/25, AMC, MATH-500, Minerva, and Olympiad) and three out-of-distribution tasks (ARC-c, GPQA-diamond, and MMLU-pro). Across noise ratios from 0.1 to 0.9, OLR consistently improves robustness under both inactive and active noisy-label settings, achieving average gains of 3.6\%–3.9\% on in-distribution benchmarks and 3.3\%–4.6\% on out-of-distribution evaluations.
The code is available via \url{https://github.com/ShenzhiYang2000/OLR}.
}}

\addsecondlogo[0.95cm]{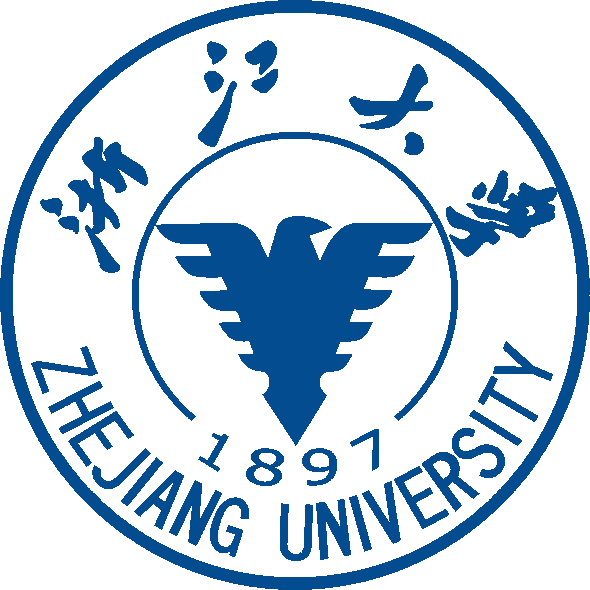} 
\begin{document}
\maketitle

\input{sections/intro}

\input{sections/related_work}

\input{sections/method}

\input{sections/experiments}
\input{sections/conclusion}

\bibliographystyle{assets/plainnat}
\bibliography{main}

\newpage
\appendix

\input{sections/appendix}

\end{document}

%% file: sections/intro.tex
\section{Introduction}


\begin{wrapfigure}{!t}{0.6\textwidth}
	\centering
 \includegraphics[width=1.0\linewidth]{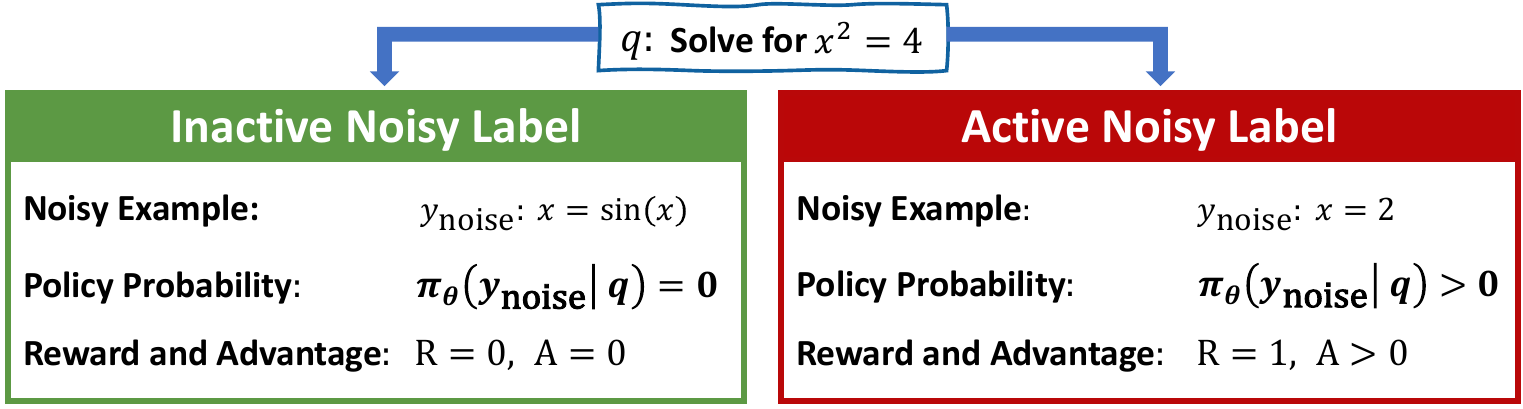}
  \caption{\textbf{Two types of noisy labels in RLVR.} \textit{Inactive noisy label} (left): an incorrect initial label that the model cannot trigger the corresponding reasoning path for and reinforce, so it remains inactive.  
\textit{Active noisy label} (right): an incorrect label that the model has a probability of triggering the corresponding reasoning path for and reinforcing, so it remains active.}
  \label{fig:NoisySample}
\end{wrapfigure}

Following the success of DeepSeek-R1 \citep{guo2025deepseek}, reinforcement learning with verifiable rewards (RLVR) has rapidly emerged as a powerful paradigm for training reasoning-oriented language models. While algorithms such as GRPO \citep{grpo} and its variants \citep{drgrpo,zheng2025group,yu2025dapo} eliminate the need for critic models via intra-group advantage estimation, this framework relies on large amounts of carefully curated data. In practice, however, expert scarcity and weak verifiers \citep{yan2025verifybench} make imperfect labeling unavoidable, resulting in noisy training labels.
Although noisy label learning \citep{natarajan2013learning,sukhbaatar2014training} has been extensively studied in traditional supervised classification \citep{hendrycks2018using,collier2021correlated,li2024noisy}, these studies predominantly focus on static datasets and limited output label spaces. 
However, this perspective overlooks a critical characteristic of most RLVR algorithms arising from its near-infinite label space and on-policy data generation: \textbf{a label's influence on training is contingent on whether the current policy can generate rollouts that realize it, a property that naturally extends to noisy labels.}
Unlike traditional classification, where all wrong labels contribute to loss, in RLVR, a noisy label is reinforced only if the model can generate it. If the model cannot generate the rollouts to realize the noisy label (\textbf{inactive noisy labels}), all generations have equal advantage. These do not actively mislead but waste rollouts and reduce data efficiency. More harmful are noisy labels the model can generate (\textbf{active noisy labels}), which receive positive advantage and steer the policy toward incorrect distributions (Fig.\ref{fig:NoisySample}).
To date, robust training of large language models (LLMs) for reasoning under such noisy supervision remains largely unexplored.

In this work, we take the first step toward a systematic analysis of
noisy-label learning in RLVR.
We observe a key training dynamic, which we term the \textbf{Early Correctness Coherence}: \textbf{\textit{although noisy samples begin to lag behind in later stages, accuracy on both clean and noisy samples increases similarly in early training.}} Interestingly, even with noisy supervision, the chance that the majority answer \citep{zuo2025ttrl} matches the correct answer increases. \textit{This indicates that the training process already produces real correct answers for noisy samples }(Fig.~\ref{fig:earlyfitting}).
Motivated by this observation, we propose \textbf{O}nline \textbf{L}abel \textbf{R}efinement (\textbf{OLR}). Instead of relying on externally provided labels, OLR selectively corrects potentially noisy labels, whether inactive or active, with the model’s own policy-generated majority answers, which are more likely to reflect these emerging correct answers.
Intuitively, this resembles a student learning: early coherence learning builds reasoning ability, and answers that repeatedly emerge from the student’s own practice with increasing confidence are likely to be correct.
In our method, we formalize this intuition by monitoring two signals:

\emph{\textbf{1.} The {\texttt{slope}} of the majority answer’s pass rate across rollouts. A positive slope indicates that as
   the model repeatedly attempts the problem, more rollouts converge to the same answer, suggesting it
   provides a positive expected advantage and meaningful learning signal as the policy improves.}
   
\emph{\textbf{2.} {\texttt{Historical consistency}} of the majority answer. By tracking whether the same majority answer remains dominant over consecutive updates, we filter out transient or accidental majority answers, preventing them from misleading the model.}

In summary, once a majority answer exhibits both a positive slope and historical consistency, we replace the original label with it. This replacement is guided by two complementary signals: the model's improving capability (evidenced by convergent rollouts) and reinforcement dynamics (the answer consistently yields higher expected advantage). As reasoning improves and correct majorities emerge, noisy labels are gradually overwritten, enabling the policy to continuously refine its own training targets. We evaluate OLR on six mathematical reasoning benchmarks (AIME24/25, AMC, MATH-500, Minerva, and Olympiad) and three out-of-distribution tasks (ARC-c, GPQA-diamond, and MMLU-Pro). Across noise ratios from 0.1 to 0.9, OLR achieves consistent gains. Under inactive noise, it yields average improvements of \textbf{3.6\%} on in-distribution benchmarks and \textbf{3.3\%} on out-of-distribution tasks. Under active noise, the gains increase to \textbf{3.9\%} and \textbf{4.6\%}, respectively. These results demonstrate that OLR robustly handles both noise types across diverse task distributions.

%% file: sections/related_work.tex
\section{Related Work}
\noindent\textbf{Reinforcement Learning with Verifiable Rewards (RLVR)} has proven highly effective for training reasoning models in domains with clear ground truth, such as math and code~\citep{orz, team2025kimi, mroueh2025reinforcement, li2025limr}. By replacing learned reward models with rule-based verifiers~\citep{jaech2024openai, guo2025deepseek}, RLVR avoids the complexities of human preference modeling~\citep{christiano2017deep, ouyang2022training}, enabling stable training pipelines that have produced powerful models like DeepSeek-R1~\citep{guo2025deepseek} via algorithms such as GRPO~\citep{grpo, drgrpo, yu2025dapo, zheng2025group}. Parallel work has investigated alternative unsupervised methods, including self-judgment~\citep{wu2024meta,yuan2024self,xiong2025self}, ensemble heads~\citep{wang2024cream,zhou2025self}, and heuristics like entropy and majority voting~\citep{agarwal2025unreasonable,li2025confidence,zuo2025ttrl}, to enable scalable online learning~\citep{zhang2025right,zhao2025learning}.
However, despite the growing body of RLVR research, the critical issue of learning under noisy labels remains largely unaddressed. This paper presents a first exploration of noisy label mitigation in RLVR.

\noindent \textbf{Noisy Label Learning} is well-established, but its techniques rarely accommodate RLVR's generation and on-policy dynamics. Existing methods fall into three categories. 1) \textit{Noise Transition Matrix Estimation} \citep{xiao2015learning,chen2015webly,srivastava2014dropout,sukhbaatar2014training,collier2021correlated,bucarelli2023leveraging} learns flipping probabilities between labels, but is unsuitable for RLVR, where noisy labels form an open set without clear class boundaries. 2) \textit{Loss Correction} \citep{goodfellow2014explaining,pereyra2017regularizing,zhang2017mixup,menon2020can,cheng2021mitigating} uses regularization to reduce overfitting to noise, but these classification-focused methods have limited applicability to RLVR's generative tasks. 3) \textit{Small-loss-based Sample Selection} \citep{malach2017decoupling,wang2018iterative,han2018co,jiang2018mentornet,wu2021ngc,song2021robust} treats low-loss samples as clean. However, in RLVR, small-loss samples are not necessarily beneficial \citep{zhan2025exgrpo}.

%% file: sections/method.tex
\begin{figure*}[!t]
  \centering

    \begin{subfigure}{0.24\textwidth}
    \centering
    \includegraphics[width=\linewidth]{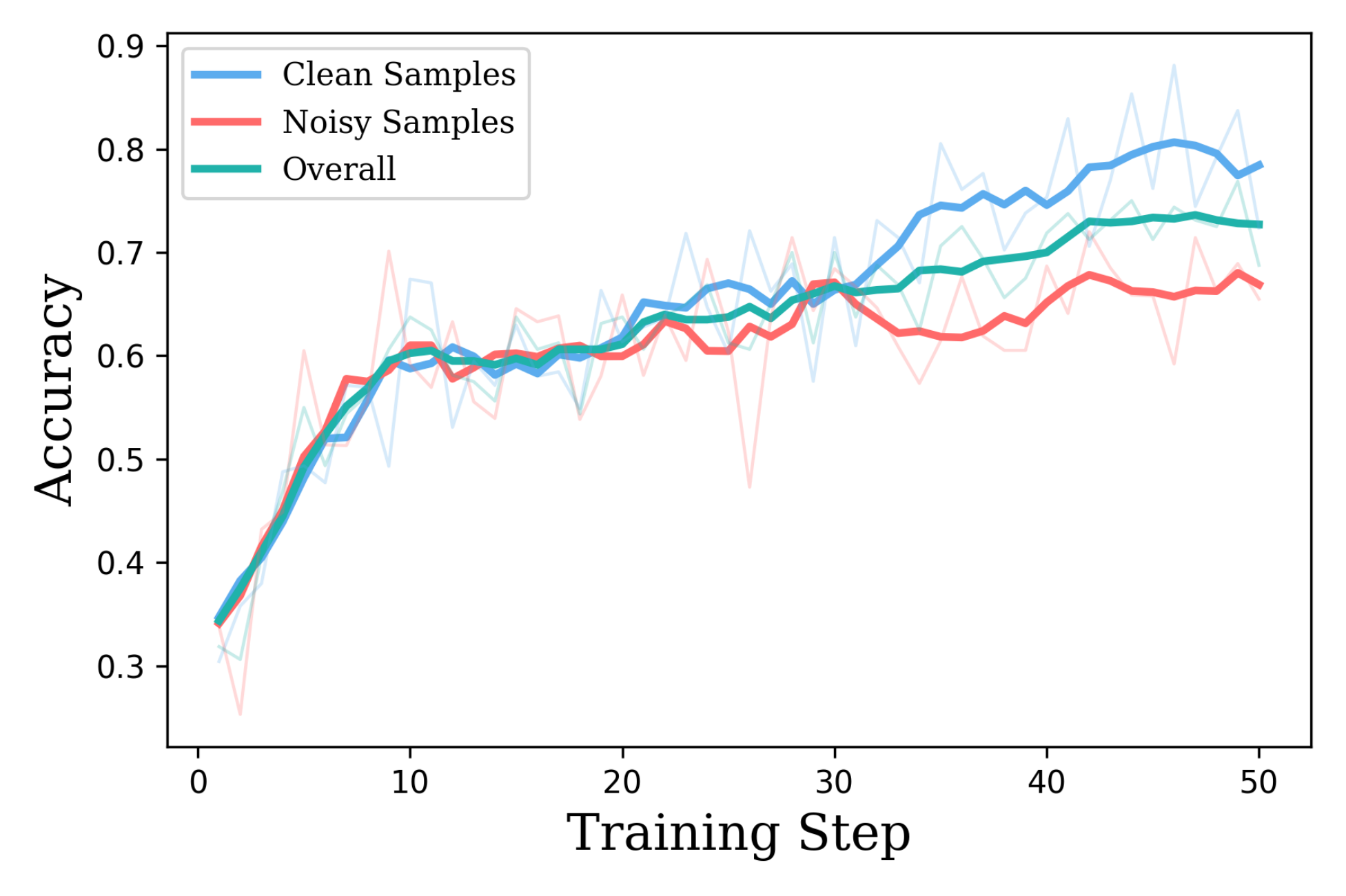}
    \caption{Inactive Noise w/o OLR}
  \end{subfigure}
  \hfill
    \begin{subfigure}{0.24\textwidth}
    \centering
    \includegraphics[width=\linewidth]{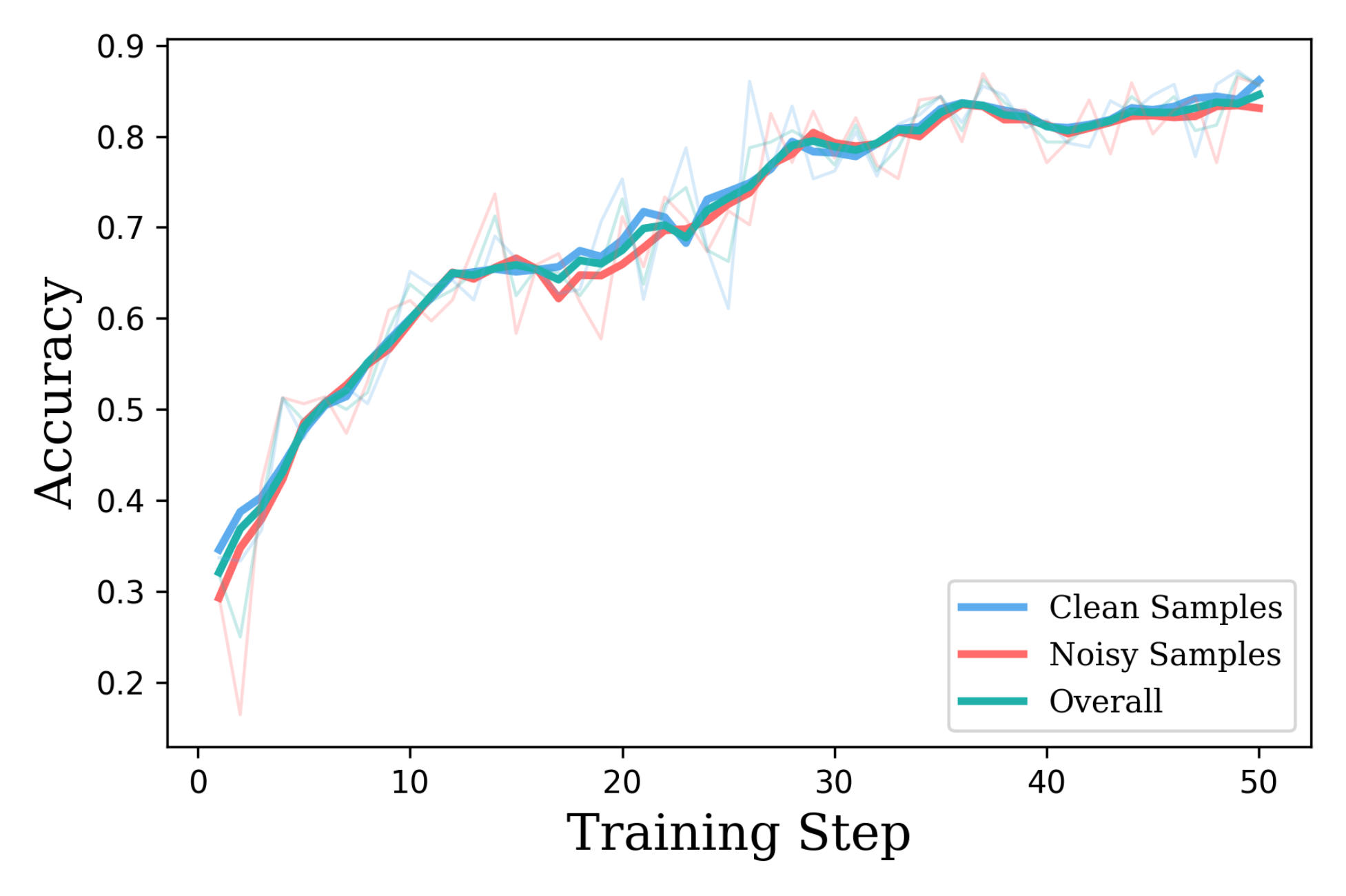}
    \caption{Inactive Noise w/ OLR}
  \end{subfigure}
  \hfill
  \begin{subfigure}{0.24\textwidth}
    \centering
    \includegraphics[width=\linewidth]{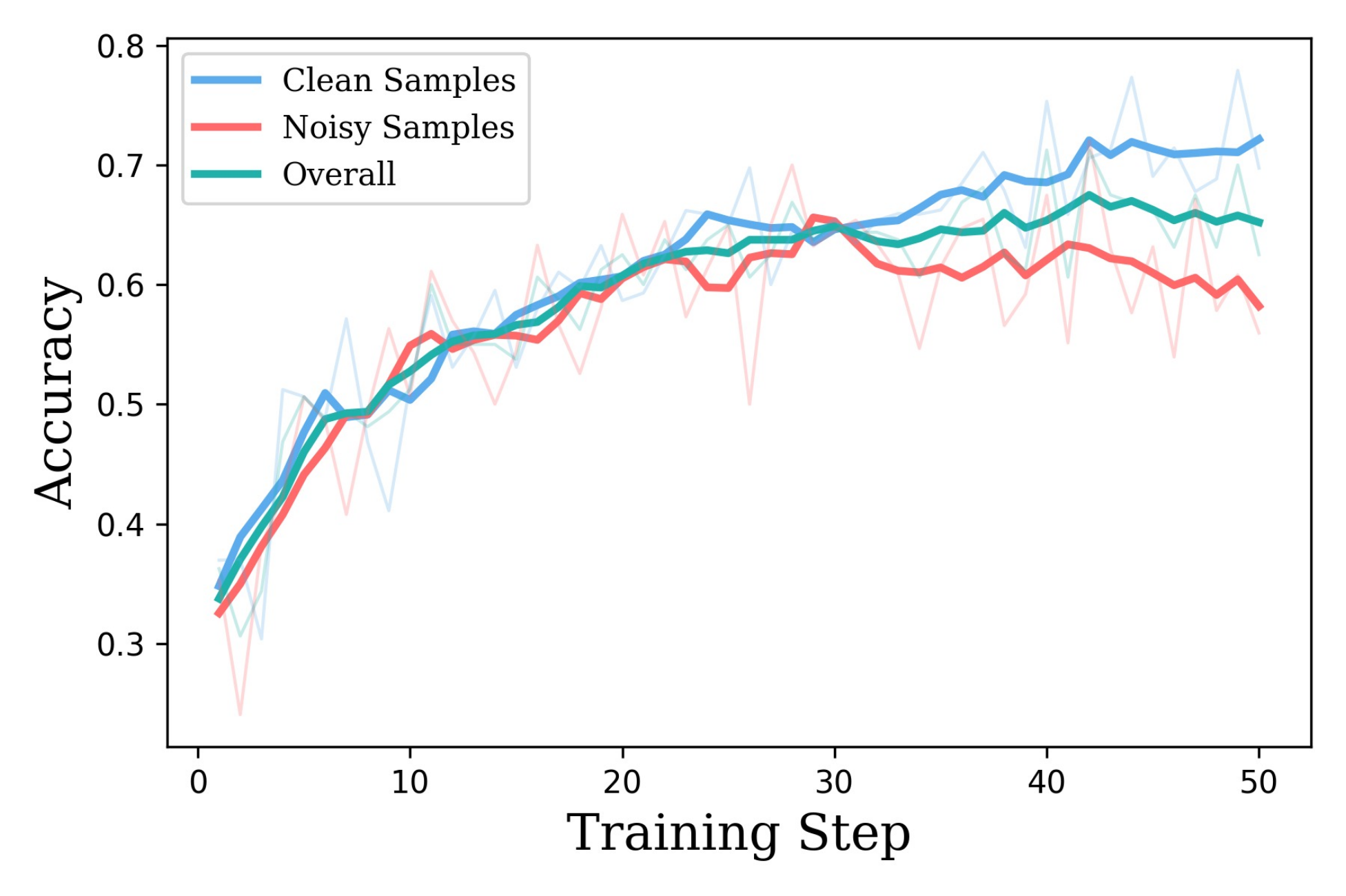}
    \caption{Active Noise w/o OLR}
  \end{subfigure}
  \hfill
  \begin{subfigure}{0.24\textwidth}
    \centering
    \includegraphics[width=\linewidth]{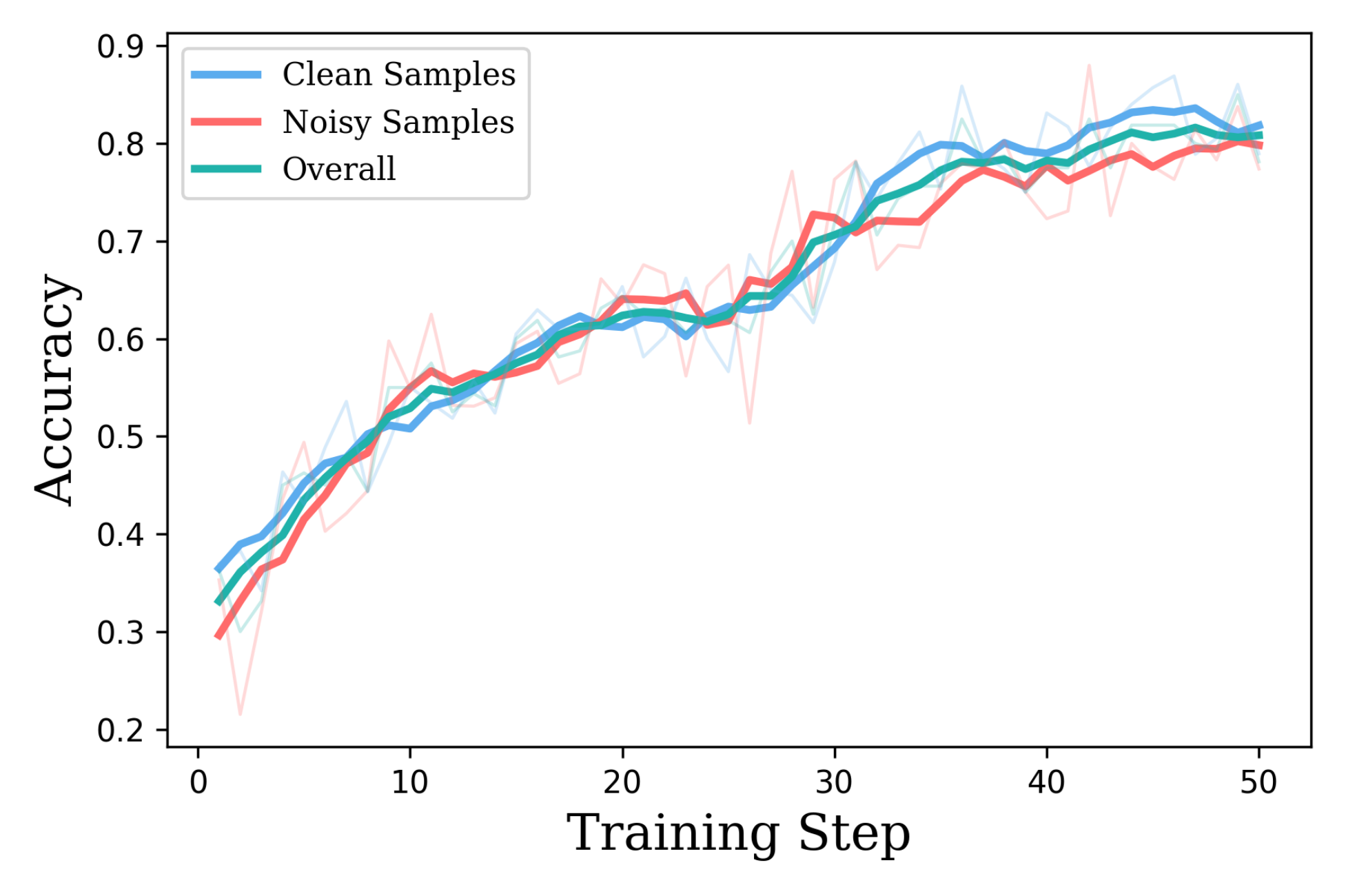}
    \caption{Active Noise w/ OLR}
  \end{subfigure}

  \caption{\textbf{Early Correctness Coherence.} Training accuracy of Qwen3-4B-Base \citep{yang2025qwen3} under noisy supervision (noise ratio 0.5). \textbf{For each sample, we take the majority answer as the model's prediction.} Clean and noisy samples exhibit similar learning dynamics early in training but gradually diverge as training progresses. \textbf{(i)} In the initial phase, the accuracy of correct answers from both groups \textbf{increases steadily}, suggesting that \textit{the model already contains latent correct answers for noisy samples that are not fully exploited}. \textbf{(ii)} In later stages, accuracy on clean samples continues to improve while performance on noisy samples lags behind. Our method, \textbf{O}nline \textbf{L}abel \textbf{R}efinement (\textbf{OLR}), utilizes this early coherence and significantly improves reasoning performance.
}
  \label{fig:earlyfitting}
\end{figure*}

\section{Method}\label{sec:Method}
Section \ref{subsec:problem_formulation} formulates the standard RLVR problem and provides the first definition of noisy labels in RLVR. Section \ref{subsec:Correctness-Early-theory} presents a theoretical analysis of the Early Correctness Coherence phenomenon observed in practice. Building on this analysis, Section \ref{subsec:OLR} introduces our proposed method, \textbf{O}nline \textbf{L}abel \textbf{R}efinement (\textbf{OLR}), the first noisy label mitigation approach for RLVR.

\subsection{Problem Formulation}\label{subsec:problem_formulation}

We consider reinforcement learning with verifiable rewards (RLVR), where a policy
$\pi_\theta(y \mid x)$ is trained to generate a solution $y$ for a given input prompt $x$.
Unlike supervised classification, training data are generated on-policy: for each prompt $x$,
the model samples $K$ rollouts
$
\mathcal{Y}(x) = \{y^{(1)}, \dots, y^{(K)}\} \sim \pi_\theta(\cdot \mid x).
$
A verifier provides a scalar reward signal
$
r(x, y) \in \mathbb{R},
$
which is typically binary, indicating whether the generated solution satisfies
a certain criterion (e.g., correctness or format adherence).
Among the relevant methods, the most popular is Group Relative Policy Optimization (GRPO) \citep{grpo}, which computes the advantage for each
rollout by normalizing rewards within the rollout group:
\begin{equation}
A(x, y^{(k)}) =
\frac{r(x, y^{(k)}) - \mu(x)}{\sigma(x) + \epsilon},
\end{equation}
where
$
\mu(x) = \frac{1}{K} \sum_{k=1}^K r(x, y^{(k)}), \quad
\sigma(x) = \sqrt{\frac{1}{K} \sum_{k=1}^K (r(x, y^{(k)}) - \mu(x))^2},
$ $\epsilon$ is the minimum value to avoid division by zero.
The policy is updated by maximizing the expected advantage-weighted log-likelihood:
\begin{equation}
\mathcal{L}_{\text{GRPO}}(\theta) = \mathbb{E}_{x} \mathbb{E}_{y \sim \pi_\theta(\cdot \mid x)} 
\left[ \text{CLIP}(A(x, y) \frac{\pi_{\theta}(y \mid x)}{\pi_{\theta_{\text{old}}}(y \mid x)}, \epsilon) - \beta\cdot\mathbb{D}_{\text{KL}}[\pi_\theta \| \pi_{\text{ref}}]\right]
\end{equation}
where $\text{CLIP}$ is the clipped surrogate objective, and $\mathbb{D}_{\text{KL}}$ 
denotes the KL divergence.

Although RLVR algorithms like GRPO have achieved great success, they typically require large amounts of perfectly annotated data. In practice, however, due to the scarcity of experts, noisy labels are unavoidable. Moreover, despite their prevalence, in-depth analysis of noisy labels in RLVR is almost nonexistent. Therefore, we take the first step toward understanding their mechanism. To establish a formal understanding, we first define noisy labels within the RLVR framework. Specifically, for a prompt $x$, let $y^\star(x)$ denote the ground-truth solution and $\tilde{y}(x)$ the potentially corrupted training label. Furthermore, whether a noisy label will be reinforced depends critically on whether the policy can roll it out, and we define this property as Rollout Feasibility as follows:
\begin{definition}[Rollout Feasibility]
A solution $y$ is \emph{rollout-feasible} under $\pi_\theta$ if $\pi_\theta(y\mid x) > 0$.
\end{definition}
Then, based on whether the noisy label is rollout-feasible, we classify it into the following two distinct types of noisy labels\footnote{It is worth noting that the correct labels for all samples should be rollout-feasible. If this condition is not met, the sample is invalid and requires external knowledge for the model to learn it, though this falls outside the scope of our research.}:
\begin{definition}[Inactive Noisy Label]
$\tilde y(x) \neq y^\star(x)$ is \emph{inactive} if it is not rollout-feasible:
$\pi_\theta(\tilde y \mid x) = 0$.
\end{definition}

\begin{definition}[Active Noisy Label]
$\tilde y(x) \neq y^\star(x)$ is \emph{active} if it is rollout-feasible:
$\pi_\theta(\tilde y \mid x) > 0$.
\end{definition}

With this formalization in place, a fundamental question arises: \textbf{can models learn to reason robustly under such noisy supervision?} To answer this question, we analyze the training behavior of noisy RLVR 
and the phenomenon, Early Correctness Coherence, observed in practice.

\subsection{Early Correctness Coherence Phenomenon in Noisy RLVR}
\label{subsec:Correctness-Early-theory}

As shown in Figure \ref{fig:earlyfitting}, we observe the phenomenon of \textbf{Early Correctness Coherence} in noisy RLVR training experiments: \textbf{\textit{although noisy samples begin to lag behind in later stages, accuracy on both clean and noisy samples increases similarly in early training.}} 
This intriguing phenomenon reveals that, despite noisy supervision, the training process already produces real correct answers for noisy samples.
We attribute this phenomenon to positive cross-sample coupling enabled by shared parameters, where updates from clean samples inadvertently benefit noisy ones. This mechanism explains why both groups improve in unison during early training, even under noisy supervision.
We formalize this mechanism below.

\begin{theorem}[Early Correctness Coherence in Noisy RLVR] (Informal)
\label{thm:early-main-final}
Let $\mathcal D=\mathcal D_{\rm clean}\cup\mathcal D_{\rm noise}$ with noise ratio
$\rho = |\mathcal D_{\rm noise}|/|\mathcal D|$.  
For each prompt $x$ at epoch $t$, let $p_t(y|x)=\pi_{\theta_t}(y|x)$ and define log-ratio
$
L_t(x)=\log\frac{p_t(y^\star(x)|x)}{p_t(\tilde y(x)|x)} .
$
Assume:  
(i) $p_0(y^\star)>p_0(\tilde y)$;  
(ii) positive cross-sample coupling between the clean sample $x_c$ and the noisy sample $x_n$:
\begin{equation}
\mathbb E_{x_c,x_n}\Big[
\nabla_\theta \log \pi(y^\star(x_c)|x_c)  
\cdot \nabla_\theta \log \pi(y^\star(x_n)|x_n)
\Big] \ge \gamma > 0 
\end{equation}
(iii) mean drift
$
\Delta_s = \gamma(1-\rho)G_c - \rho G_n
$, where $G_c$ and $G_n$ are the mean clean and noisy advantage magnitudes.
Let
$
\rho_c=\frac{\gamma G_c}{\gamma G_c+G_n}.
$
If $\rho<\rho_c$ and $K\gtrsim\log(T/\delta)$, then with probability at least
$1-\delta$, for all $t\le T$,
\begin{equation}
L_t(x)\ge
L_0(x)+
\eta t\!\left(
\Delta_s -
O\!\left(\sqrt{\frac{\log(T/\delta)}{K}}\right)
\right),
\end{equation}
implying that
$
p_t(y^\star(x)|x) \gg p_t(\tilde y(x)|x).
$
\end{theorem}

See Appendix~\ref{app:technical} for the formal Theorem~\ref{thm:early-main-final-formal} and the full proof, and Appendix Figure~\ref{fig:assumption_verify} for experimental verification of Theorem~\ref{thm:early-main-final}'s assumptions.

\begin{remark}
Theorem~\ref{thm:early-main-final} suggests that in the early phase, \textit{\textbf{correct answers, with increasing probability, can gradually emerge even on noisy samples through cross-sample coupling.}}
\end{remark}
This indicates that correct answers may be recovered from the model’s rollouts as training progresses, motivating our method, which leverages these emerging answers for label refinement.

\subsection{Online Label Refinement (OLR)}\label{subsec:OLR}

The analysis of Early Correctness Coherence above suggests that correct solutions tend to emerge with progressively higher probability during early training, even on noisy samples. 
Thus, answers whose rollout probability consistently increases are likely to be correct; if such answers differ from the provided label, the original supervision is likely noisy and can be replaced, which forms the foundation of our \textbf{O}nline \textbf{L}abel \textbf{R}efinement (\textbf{OLR}).

Formally, 
at epoch $t$, the policy generates $K$ rollouts 
$\mathcal{Y}_t(x) = \{y^{(1)}_t, \dots, y^{(K)}_t\} \sim \pi_\theta(\cdot \mid x)$.
To estimate solution probabilities, we use the empirical pass rate from these rollouts. 
Since enumerating all possible answers is infeasible, we track only the \textit{majority answer}
$
y^{\text{maj}}_t(x) = \arg\max_{c} |\{y \in \mathcal{Y}_t(x) : y = c\}|,
$
with pass rate
$
p^{\text{maj}}_t(x) = \frac{1}{K}|\{y \in \mathcal{Y}_t(x) : y = y^{\text{maj}}_t(x)\}|.
$
We maintain a trajectory
$
\mathcal{H}_t(x) = \{(t', y^{\text{maj}}_{t'}(x), p^{\text{maj}}_{t'}(x)) : t' \in [1,t]\}
$
to evaluate the reliability of the candidate answer.

\textbf{Criterion 1: Positive Convergence Slope.}
To detect reliable improvement, we compute the linear regression slope of the pass-rate trajectory. Let
$\mathbf{p}_t(x) = [p^{\text{maj}}_{1}(x), \dots, p^{\text{maj}}_t(x)]^\top$
and $\mathbf{t} = [1, \dots, t]^\top$. The slope is
\begin{equation}\label{eq:slope}
S_t(x)=
\frac{(\mathbf{t}-\bar{t}\mathbf{1})^\top(\mathbf{p}_t(x)-\bar{p}\mathbf{1})}
{(\mathbf{t}-\bar{t}\mathbf{1})^\top(\mathbf{t}-\bar{t}\mathbf{1})},
\end{equation}
where $\bar{t}=\frac{1}{t}\sum_{i=1}^t i$ and $\bar{p}=\frac{1}{t}\sum_{i=1}^t p^{\text{maj}}_{i}(x)$. 
A positive slope $S_t(x)>\delta_{\text{slope}}$ indicates increasing model confidence.

\textbf{Criterion 2: Historical Consistency.}
To filter stochastic fluctuations, we check whether the current majority answer matches the historical majority
$
y^{\text{hist}}_t(x)=
\arg\max_{y}|\{(t',y^{\text{maj}}_{t'}(x))\in\mathcal{H}_t(x):y^{\text{maj}}_{t'}(x)=y\}|.
$
The consistency indicator is
\begin{equation} \label{eq:hist_consistency}
C_t(x)=\mathbb{I}\left(y^{\text{maj}}_t(x)=y^{\text{hist}}_t(x)\right).
\end{equation}

After an initial early learning phase of $T$ epochs to accumulate rollout statistics, the effective label $\hat{y}_t(x)$ used for reward computation is
\begin{equation}
\hat{y}_t(x)=
\begin{cases}
y^{\text{maj}}_t(x) & \text{if } S_t(x)>\delta_{\text{slope}} \text{ and } C_t(x),\\
\tilde{y}(x) & \text{otherwise}.
\end{cases}
\label{eq:effective_label}
\end{equation}

We adopt GRPO as the base optimizer; incorporating OLR enhances the model's tolerance to noisy labels during training. See below for theoretical analysis and Fig.~\ref{fig:sensitive} (f) for experimental validation.
\begin{theorem}[OLR Improves Label Noise Tolerance]
\label{thm:olr-tolerance}
Let $\Delta$ denote the probability that a noisy prompt satisfies the OLR replacement criteria. 
When OLR replaces a label, the selected label equals the correct solution with probability at least
$
\Pr(\hat y_t(x)=y^\star(x)) \ge 1-\epsilon,
$
where $\epsilon = O(\exp(-K\Delta_p^2))$ and $\Delta_p$ denotes the probability gap between the correct solution and competing answers.
Consequently, the effective noise ratio becomes
$
\rho_{\text{eff}} = \rho(1-\Delta) < \rho,
$
which increases the tolerable noise threshold from $\rho_c$ to
$
\rho_c^{\text{OLR}} = \frac{\rho_c}{1-\Delta} > \rho_c
$
\end{theorem}
For a detailed proof, please refer to Appendix \ref{proof_thm_noise_tolerance}.
Overall, OLR dynamically selects reliable majority answers to rewrite noisy labels by analyzing the slope of the pass rate trend for majority answers and their historical consistency, thereby achieving effective denoising. Compared to model inference and update, the additional time overhead introduced by OLR is almost negligible (see Appendix Table~\ref{tab:olr_time}). Below, we demonstrate OLR's effectiveness through experimental results.

%% file: sections/experiments.tex
\begin{table*}[!t]
\caption{Results on Qwen3-4B-Base under different noise ratios. Each numerical subscript indicates the absolute \textcolor{lightgreen}{\textbf{improvement}} or \textcolor{myred}{\textbf{degradation}} compared with naive GRPO without OLR. \textbf{Bold} indicates the better.}
\label{tab:qwen4b_trapo_results}
\centering
\small
\resizebox{\textwidth}{!}{%
\begin{tabular}{lccccccccccc}
\toprule
\multirow{2.5}{*}{\textbf{Method}} 
& \multicolumn{7}{c}{\textbf{In-Distribution}} 
& \multicolumn{4}{c}{\textbf{Out-of-Distribution}} \\
\cmidrule(lr){2-8}
\cmidrule(lr){9-12}
& AIME24 & AIME25 & AMC & MATH-500 & Minerva & Olympiad 
& \phantom{0}\phantom{0}\textbf{Avg.}\phantom{0}\phantom{0}
& ARC-c & GPQA$^{\star}$ & MMLU-Pro &  \phantom{0}\phantom{0}\textbf{Avg.}\phantom{0}\phantom{0}\\
\midrule

\makecell[l]{Qwen3-4B-Base} 
& 9.6 & 4.2 & 34.2 & 52.6 & 19.5 & 28.4 & 24.8 
& 35.8 & 14.1 & 33.3 & 27.7\\

\midrule
\multicolumn{12}{l}{\emph{\quad \textbf{Inactive Noisy Label}}} \\

$\rho = 0.1$
& 21.3 & 20.4 & 52.7 & 83.6 & 39.7 & 49.3 & 44.5 
& 86.3 & 37.4 & 59.0 & 60.9 \\

\rowcolor{cyan!10}
$\ \ \ \ \text{w/ OLR}$ 
& 22.5 & 19.6 & 55.0 & 83.0 & 41.2 & 48.3 
& \cimp{\textbf{44.9}}{\textbf{+0.4}} 
& 88.9 & 37.9 & 59.5 
& \cimp{\textbf{62.1}}{\textbf{+1.2}} \\

\midrule

$\rho = 0.3$
& 12.9 & 8.3 & 50.2 & 75.8 & 42.0 & 43.0 & 38.7 
& 88.1 & 29.8 & 56.2 & 58.0 \\

\rowcolor{cyan!10}
$\ \ \ \ \text{w/ OLR}$ 
& 23.8 & 16.7 & 53.5 & 82.6 & 40.8 & 48.7 
& \cimp{\textbf{44.4}}{\textbf{+5.7}} 
& 88.1 & 36.4 & 57.6 
& \cimp{\textbf{60.7}}{\textbf{+2.7}} \\

\midrule

$\rho = 0.5$
& 12.5 & 5.0 & 45.3 & 76.0 & 34.9 & 42.7 & 36.1 
& 86.7 & 28.8 & 56.0 & 57.2 \\

\rowcolor{cyan!10}
$\ \ \ \ \text{w/ OLR}$ 
& 21.7 & 18.3 & 53.9 & 84.2 & 42.3 & 48.9 
& \cimp{\textbf{44.9}}{\textbf{+8.8}} 
& 86.8 & 38.4 & 57.8 
& \cimp{\textbf{61.0}}{\textbf{+3.8}} \\

\midrule

$\rho = 0.7$
& 9.6 & 8.8 & 45.9 & 76.4 & 37.5 & 42.7 & 36.8 
& 83.4 & 26.3 & 54.5 & 54.7 \\

\rowcolor{cyan!10}
$\ \ \ \ \text{w/ OLR}$  
& 14.6 & 11.3 & 48.0 & 79.2 & 37.1 & 43.4 
& \cimp{\textbf{38.9}}{\textbf{+2.1}} 
& 87.4 & 29.8 & 57.3 
& \cimp{\textbf{58.2}}{\textbf{+3.5}} \\

\midrule

$\rho = 0.9$
& 12.5 & 9.6 & 44.3 & 73.4 & 37.5 & 40.0 & 36.2 
& 76.5 & 27.3 & 51.5 & 51.8 \\

\rowcolor{cyan!10}
$\ \ \ \ \text{w/ OLR}$ 
& 10.8 & 9.6 & 48.2 & 76.6 & 37.1 & 42.2 
& \cimp{\textbf{37.4}}{\textbf{+1.2}} 
& 85.9 & 30.3 & 55.3 
& \cimp{\textbf{57.2}}{\textbf{+5.4}} \\

\midrule
\multicolumn{11}{l}{\emph{\quad \textbf{Active Noisy Label}}} \\

$\rho = 0.1$
& 13.8 & 9.6 & 48.0 & 77.8 & 37.5 & 42.4 & 38.2 
& 84.6 & 30.3 & 55.5 & 56.8\\

\rowcolor{cyan!10}
$\ \ \ \ \text{w/ OLR}$ 
& 19.2 & 19.2 & 52.0 & 81.6 & 39.7 & 52.0 
& \cimp{\textbf{44.0}}{\textbf{+5.8}} 
& 86.8 & 35.4 & 59.1 &  \cimp{\textbf{60.4}}{\textbf{+3.6}} \\

\midrule

$\rho = 0.3$
& 13.8 & 9.6 & 47.1 & 73.6 & 36.0 & 41.8 & 37.0 
& 80.8 & 25.3 & 51.7 &52.6\\

\rowcolor{cyan!10}
$\ \ \ \ \text{w/ OLR}$  
& 13.8 & 12.9 & 50.3 & 77.4 & 36.4 & 41.3 
& \cimp{\textbf{38.7}}{\textbf{+1.7}} 
& 85.4 & 33.3 & 56.7 & \cimp{\textbf{58.5}}{\textbf{+5.9}}\\

\midrule

$\rho = 0.5$
& 10.4 & 8.8 & 45.8 & 74.6 & 33.1 & 40.4 & 35.5 
& 73.9 & 24.7 & 48.6 &49.1\\

\rowcolor{cyan!10}
$\ \ \ \ \text{w/ OLR}$ 
& 20.4 & 15.4 & 49.7 & 81.4 & 36.4 & 48.1 
& \cimp{\textbf{41.9}}{\textbf{+6.4}} 
& 78.5 & 28.3 & 54.8 & \cimp{\textbf{53.9}}{\textbf{+4.8}}\\

\midrule

$\rho = 0.7$
& 10.4 & 9.2 & 45.5 & 74.2 & 31.6 & 39.0 & 35.0 
& 66.2 & 24.7 & 49.2 &46.7\\

\rowcolor{cyan!10}
$\ \ \ \ \text{w/ OLR}$ 
& 10.8 & 10.0 & 47.1 & 76.2 & 37.5 & 41.5 
& \cimp{\textbf{37.2}}{\textbf{+2.2}} 
& 80.1 & 31.8 & 52.4 &  \cimp{\textbf{54.8}}{\textbf{+8.1}}\\

\midrule

$\rho = 0.9$
& 8.3 & 5.8 & 27.0 & 43.0 & 21.7 & 24.3 & 21.7 
& 34.2 & 12.6 & 30.5 &25.8\\

\rowcolor{cyan!10}
$\ \ \ \ \text{w/ OLR}$ 
& 7.1 & 5.8 & 35.7 & 47.8 & 25.7 & 28.6 
& \cimp{\textbf{25.1}}{\textbf{+3.4}} 
& 31.1 & 16.2 & 32.1 & \cimp{\textbf{26.5}}{\textbf{+0.7}}\\

\bottomrule
\end{tabular}
}
\end{table*}

\section{Experiment}\label{sec-exp}
In Section \ref{subsec:setup}, we introduce the dataset used in the experiments, the design for generating noise labels, and the baseline methods we compare. In Section \ref{subsec:experiment_results}, we analyze the performance gains of OLR under various noise ratios, compare it with baselines, examine training dynamics, present parameter sensitivity analysis, and ablation results. Appendix \ref{apdix:more_experiment} presents additional experiments on model sizes, LLM types, Early Correctness Coherence, and training dynamics under inactive noise.
\subsection{Setup}\label{subsec:setup}

\paragraph{Dataset and Benchmarks.}
To investigate how varying proportions of active and inactive noisy labels affect RLVR, we sample 800 instances from DAPO-Math \citep{yu2025dapo} (default setting) and inject noise at ratios of \{0.1, 0.3, 0.5, 0.7, 0.9\}. Additionally, to examine scalability, we sample 4,000 instances with a noise ratio of 0.5. For inactive noise, we replace correct labels with strings that the model is unlikely to output. However, static construction of active noise is impractical, as the probability that noisy labels actually appear during rollout depends on both the noise ratio and the model's output probability\footnote{For example, with 50\% active noise proportion (400/800 samples) and 50\% model confidence on these noisy labels, rolling out each sample 8 times gives only $(1 - 0.5^8)^{400} \approx 0.21$ probability that all noisy samples output noisy labels.}. Therefore, we dynamically construct active noise labels during inference by replacing correct labels with incorrect answers generated by the model in an on-policy manner.
For evaluation, we use six in-distribution math reasoning benchmarks: AIME 2024, AIME 2025, AMC \citep{li2024numinamath}, Minerva \citep{dataset_minerva}, OlympiadBench \citep{dataset_olympiad}, and MATH-500 \citep{dataset_math}. We report \texttt{avg@32} on AIME 2024/2025 and AMC (due to small test sets) and \texttt{pass@1} on the others. For out-of-distribution generalization, we evaluate on ARC-c \citep{arc}, GPQA-diamond \citep{gpqa} (denoted GPQA$^{*}$), and MMLU-Pro \citep{mmlu_pro}, covering open-domain reasoning, graduate-level science, and academic reasoning.

\vspace{-0.5\baselineskip}
\paragraph{Baseline Methods.}
We compare two categories of baseline methods.
The first category consists of unsupervised approaches that do not rely on ground-truth labels: (1) \textbf{TTRL} \citep{zuo2025ttrl}; (2) \textbf{Co-Reward} \citep{zhang2025co}; (3) \textbf{Self-Certainty} \citep{zhao2025learning}; (4) \textbf{Token-Level Entropy} \citep{agarwal2025unreasonable}; and (5) \textbf{Sentence-Level Entropy} \citep{agarwal2025unreasonable}.
The second category includes noise-robust learning methods and transfer-friendly regularization techniques: (6) \textbf{Confidence Penalty} \citep{pereyra2017regularizing}; (7) \textbf{Label Smoothing} \citep{lukasik2020does}; (8) \textbf{Small-loss Selection} \citep{gui2021towards}; and (9) \textbf{Random Selection}.
Detailed description is in Appendix \ref{subsec:baseline_details}.






\begin{table*}[!t]
\caption{Results on Qwen3-4B-Base under a 50\% noise ratio with different baselines. Each numerical subscript indicates the absolute \textcolor{lightgreen}{\textbf{improvement}} or \textcolor{myred}{\textbf{degradation}} compared with naive GRPO. \textbf{Bold} indicates the best in inactive noise scenarios, and \underline{underline} indicates the best in active noise scenarios.}
\label{tab:qwen4b_baseline_results}
\centering
\small
\resizebox{\textwidth}{!}{%
\begin{tabular}{lccccccccccc}
\toprule
\multirow{2.5}{*}{\textbf{Method}} 
& \multicolumn{7}{c}{\textbf{In-Distribution}} 
& \multicolumn{4}{c}{\textbf{Out-of-Distribution}} \\
\cmidrule(lr){2-8}
\cmidrule(lr){9-12}
& AIME24 & AIME25 & AMC & MATH-500 & Minerva & Olympiad 
& \phantom{0}\phantom{0}\textbf{Avg.}\phantom{0}\phantom{0}
& ARC-c & GPQA$^{\star}$ & MMLU-Pro &  \phantom{0}\phantom{0}\textbf{Avg.}\phantom{0}\phantom{0}\\
\midrule

\makecell[l]{Qwen3-4B-Base} 
& 9.6 & 4.2 & 34.2 & 52.6 & 19.5 & 28.4 & 24.8 
& 35.8 & 14.1 & 33.3 & 27.7\\

\midrule
\multicolumn{12}{l}{\emph{\quad \textbf{Unsupervised Methods}}} \\
TTRL	&12.1	&8.3	&48.2	&76.4	&36.0	&40.3	&{36.9}		&86.9	&30.3	&56.0	&\underline{57.7} \\
Co-reward	&12.1	&7.1	&44.9	&76.0	&38.9	&40.7	&{36.6}		&86.1	&29.3	&55.8	&{57.1} \\
Self-certainty	&5.8	&5.0	&28.9	&41.8	&21.7	&20.6	&{20.6}		&24.3	&12.1	&30.6	&{22.3} \\
Token-entropy	&3.8	&2.1	&22.3	&44.2	&15.8	&15.7	&{17.3}		&17.0	&7.6	&28.9	&{17.8} \\
Seq-entropy	&4.6	&2.1	&25.2	&41.0	&19.5	&19.4	&{18.6}		&18.6	&6.6	&28.4	&{17.9} \\
\midrule
\multicolumn{11}{l}{\emph{\quad \textbf{Active Noisy Label}}} \\

GRPO	&10.4	&8.8	&45.8	&74.6	&33.1	&40.4	&{35.5}		&73.9	&24.7	&48.6	&{49.1} \\
$\ \ \ \ \text{w/ Confidence Penalty}$	&12.5	&9.6	&47.0	&76.4	&34.2	&39.6	& \cimp{{36.6}}{\textbf{+1.1}} 		&82.9	&32.3	&53.3	&\cimp{{56.2}} {\textbf{+7.1}} \\
$\ \ \ \ \text{w/ Label Smoothing}$	&14.6	&7.1	&44.1	&73.2	&33.5	&38.1	&\cdec{{35.1}} {\textbf{-0.4}}		&55.5	&24.2	&47.4	&\cdec{{42.4}} {\textbf{-6.7}} \\
$\ \ \ \ \text{w/ Samll-loss Select}$	&6.7	&1.7	&22.4	&43.8	&16.5	&17.6	&\cdec{{18.1}} {\textbf{-17.4}}		&32.3	&9.6	&26.7	&\cdec{{22.9}} {\textbf{-26.2}} \\
$\ \ \ \ \text{w/ Random Select}$	&9.2	&8.8	&43.7	&74.6	&33.1	&39.6	&\cdec{{34.8}} {\textbf{-0.7}}		&45.6	&14.6	&46.3	&\cdec{{35.5}} {\textbf{-13.6}} \\
\rowcolor{cyan!10}
$\ \ \ \ \text{w/ OLR}$ 	&20.4	&15.4	&49.7	&81.4	&36.4	&48.1	&\cimp{\underline{41.9}} {\textbf{+6.4}}		&78.5	&28.3	&54.8	&\cimp{{53.9}} {\textbf{+4.8}} \\

\midrule
\multicolumn{11}{l}{\emph{\quad \textbf{Inactive Noisy Label}}} \\
GRPO	&12.5	&5.0	&45.3	&76.0	&34.9	&42.7	&{36.1}		&86.7	&28.8	&56.0	&{57.2} \\
$\ \ \ \ \text{w/ Confidence Penalty}$	&13.3	&7.9	&48.6	&76.6	&39.0	&43.1	&\cimp{{38.1}} {\textbf{+2.0}}		&87.4	&31.3	&56.6	&\cimp{{58.4}} {\textbf{+1.2}} \\
$\ \ \ \ \text{w/ Label Smoothing}$		&11.7	&7.9	&47.3	&78.8	&39.0	&42.8	&\cimp{{37.9}} {\textbf{+1.8}}		&86.5	&29.8	&57.4	&\cimp{{57.9}} {\textbf{+0.7}} \\
$\ \ \ \ \text{w/ Samll-loss Select}$	&12.1	&9.6	&46.4	&76.8	&36.4	&38.5	&\cimp{{36.6}} {\textbf{+0.5}}		&73.9	&30.3	&50.7	&\cdec{{51.6}} {\textbf{-5.6}} \\
$\ \ \ \ \text{w/ Random Select}$	&11.3	&12.1	&47.4	&77.8	&38.2	&41.3	&\cimp{{38.0}} {\textbf{+1.9}}		&78.1	&31.8	&55.0	&\cdec{{55.0}} {\textbf{-2.2}} \\
\rowcolor{cyan!10}
$\ \ \ \ \text{w/ OLR}$	&21.7	&18.3	&53.9	&84.2	&42.3	&48.9	&\cimp{\textbf{44.9}} {\textbf{+8.8}}		&86.8	&38.4	&57.8	&\cimp{\textbf{61.0}} {\textbf{+3.8}} \\

\bottomrule
\end{tabular}
}
\end{table*}

\begin{table*}[!t]
\caption{Results on Qwen3-4B-Base under a 50\% noise ratio with 4K samples. \textbf{Bold} indicates the better.}
\label{tab:large_4k}
\centering
\small
\resizebox{0.9\textwidth}{!}{%
\begin{tabular}{lccccccccccc}
\toprule
\multirow{2.5}{*}{\textbf{Method}} 
& \multicolumn{7}{c}{\textbf{In-Distribution}} 
& \multicolumn{4}{c}{\textbf{Out-of-Distribution}} \\
\cmidrule(lr){2-8}
\cmidrule(lr){9-12}
& AIME24 & AIME25 & AMC & MATH-500 & Minerva & Olympiad 
& \phantom{0}\phantom{0}\textbf{Avg.}\phantom{0}\phantom{0}
& ARC-c & GPQA$^{\star}$ & MMLU-Pro &  \phantom{0}\phantom{0}\textbf{Avg.}\phantom{0}\phantom{0}\\
\midrule
\multicolumn{11}{l}{\emph{\quad \textbf{Active Noisy Label}}} \\

w/o OLR	&12.9	&7.5	&49.4	&76.4	&39.0	&41.5	&{37.8}		&85.1	&26.8	&54.1	&{55.3} \\
w/ OLR	&20.0	&18.3	&50.0	&78.8	&37.1	&43.9	&\textbf{41.4}		&82.9	&28.8	&55.4	&\textbf{55.7} \\

\midrule
\multicolumn{11}{l}{\emph{\quad \textbf{Inactive Noisy Label}}} \\
w/o OLR	&17.9	&12.9	&50.2	&78.4	&41.5	&41.9	&{40.5}		&88.4	&33.3	&57.3	&{59.7} \\
w/ OLR	&21.7	&20.8	&55.6	&85.2	&44.9	&48.3	&\textbf{46.1}		&90.4	&38.9	&60.6	&\textbf{63.3} \\
\bottomrule
\end{tabular}
}
\end{table*}

\subsection{Experimental Results}\label{subsec:experiment_results}

\paragraph{How does OLR perform under different noise ratios?}
The results in Table \ref{tab:qwen4b_trapo_results} demonstrate that OLR consistently improves post-training model performance across a range of noise ratios (0.1, 0.3, 0.5, 0.7, and 0.9) under both active and inactive noise labeling scenarios. Under inactive noise, OLR achieves average gains of \textbf{3.6}\% on six in-distribution mathematical benchmarks and \textbf{3.3}\% on three out-of-distribution benchmarks. \textit{Notably, even at a high noise ratio of 50\%, OLR maintains performance comparable to that observed at 10\% noise}, highlighting its robustness in inactive noise settings. In contrast, active noise labeling leads to a sharper decline in performance as noise increases, with the model collapsing and final performance falling below that of the initial model. In this more challenging active noise scenario, OLR still yields substantial improvements: average gains of \textbf{3.9}\% on the mathematical benchmarks and \textbf{4.6}\% on the out-of-distribution benchmarks. These results demonstrate that OLR not only enhances robustness under inactive noise but also delivers significant gains in complex active noise environments.

\paragraph{How does OLR perform compared to baseline methods?}
We present OLR's performance relative to baseline methods in Table \ref{tab:qwen4b_baseline_results} under a 50\% noise-ratio scenario from two perspectives: unsupervised methods and traditional noisy label learning approaches.
\textbf{(i) Unsupervised methods.} Most unsupervised methods lead to severe model collapse. The best among them, TTRL and Co-Reward, only marginally outperform naive noisy training on in-distribution tasks, suggesting that abandoning labels entirely is suboptimal. On out-of-distribution tasks, TTRL shows stronger robustness to active noise. Nevertheless, OLR achieves superior overall performance: average gains of \textbf{5.0}\% and \textbf{8.0}\% over unsupervised SOTA in ID settings under active and inactive noise, respectively, and outperforms the best unsupervised method by \textbf{3.3}\% on OOD tasks.
\textbf{(ii) Traditional noisy label learning methods.} We also find most noisy label methods ineffective in RLVR. Among active noisy label settings, only Confidence Penalty improves over naive training; others harm performance. Notably, small-loss selection declines most sharply, underperforming both baseline and random selection. This is because in RLVR, small loss frequently signals uninformative samples (e.g., rollouts entirely correct or incorrect) rather than high label quality, offering little improvement signal. In ID benchmarks, OLR outperforms the best of these methods by an average of \textbf{6.1}\%.

\begin{figure*}[!t]
  \centering

    \begin{subfigure}{0.32\textwidth}
    \centering
    \includegraphics[width=\linewidth]{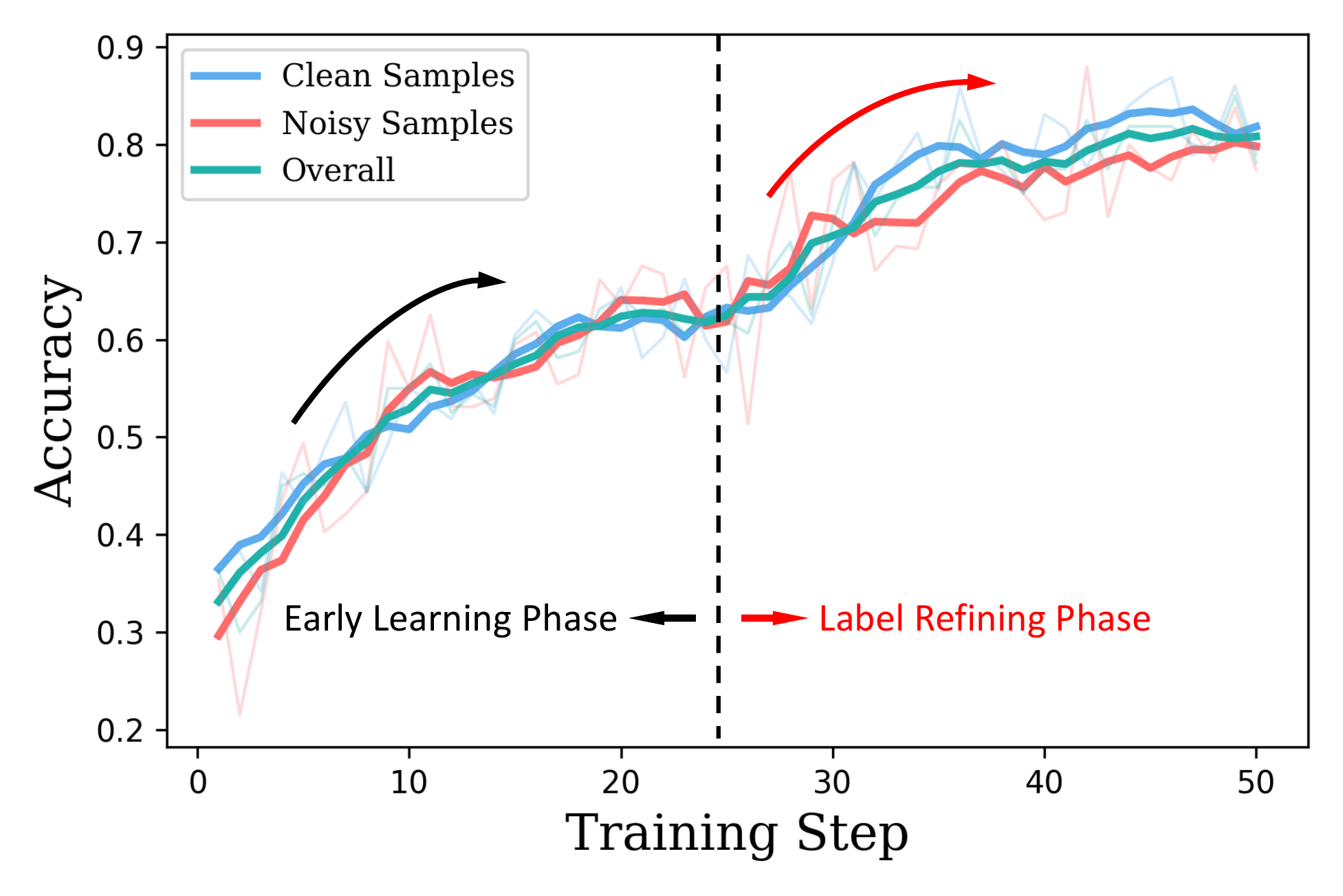}
    \caption{Training Accuracy}
  \end{subfigure}
  \hfill
    \begin{subfigure}{0.32\textwidth}
    \centering
    \includegraphics[width=\linewidth]{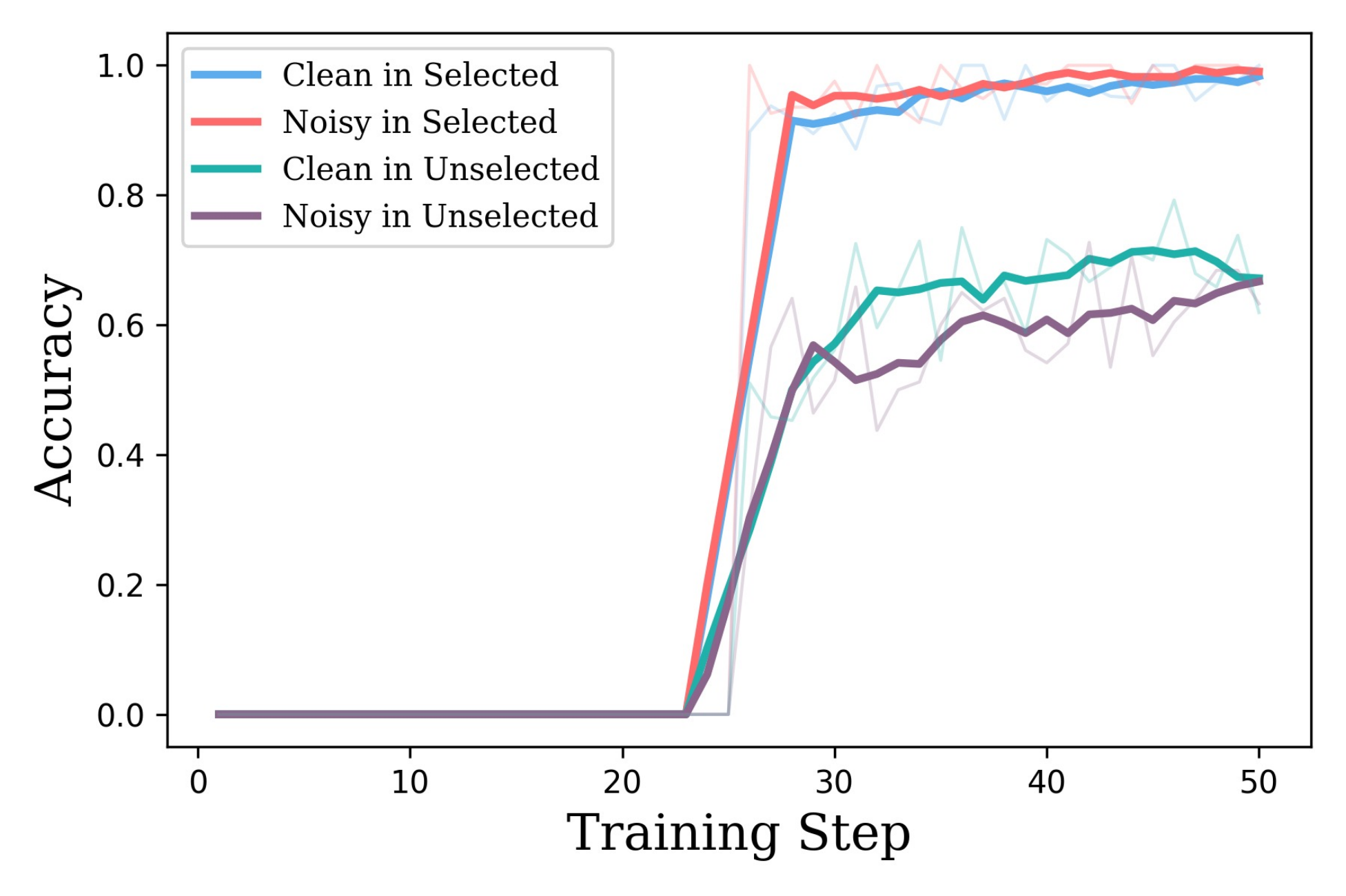}
    \caption{Selected vs. Unselected Accuracy}
  \end{subfigure}
  \hfill
  \begin{subfigure}{0.32\textwidth}
    \centering
    \includegraphics[width=\linewidth]{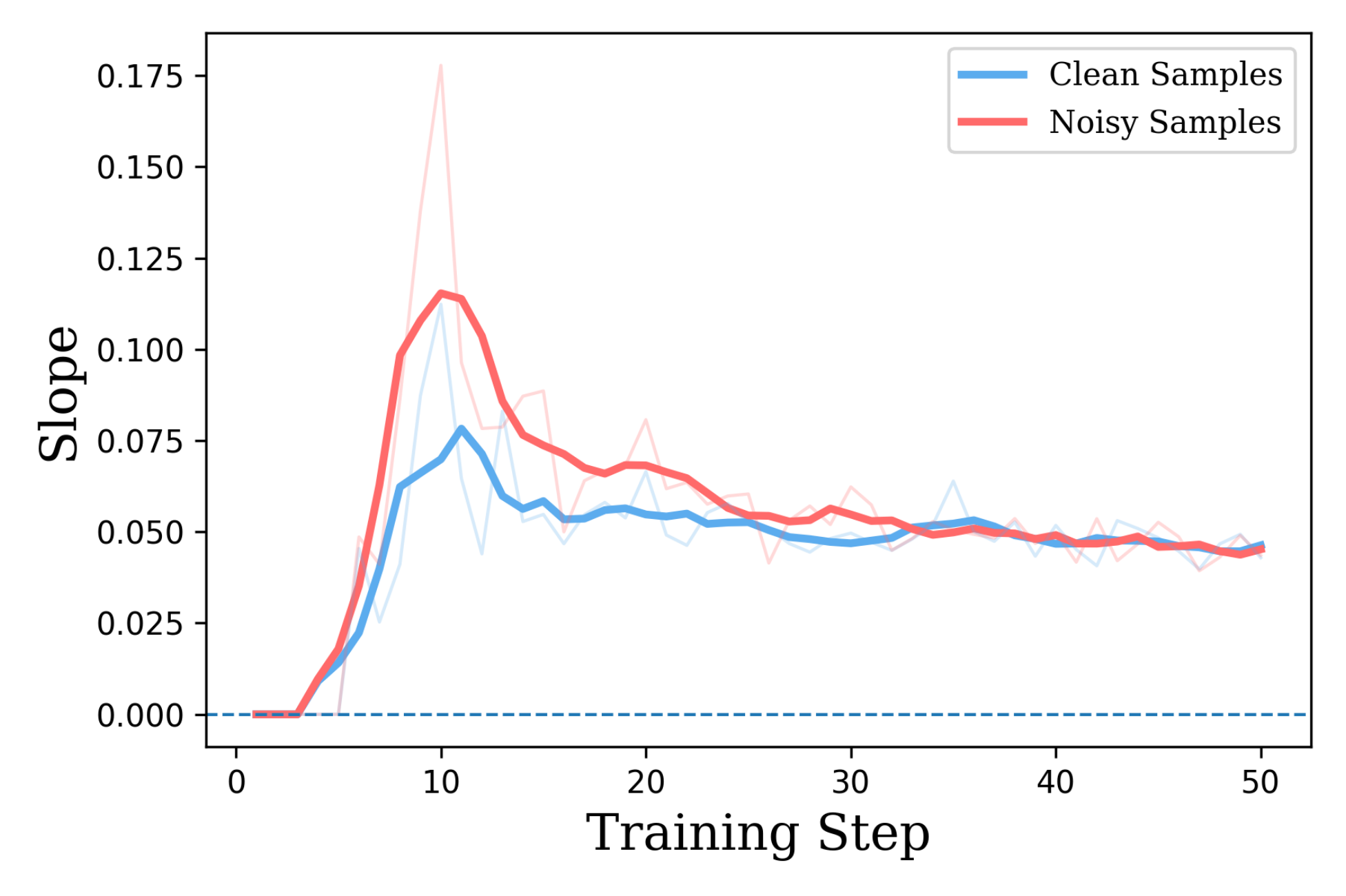}
    \caption{Passrate Slope}
  \end{subfigure}

  \vspace{2mm}

  \begin{subfigure}{0.32\textwidth}
    \centering
    \includegraphics[width=\linewidth]{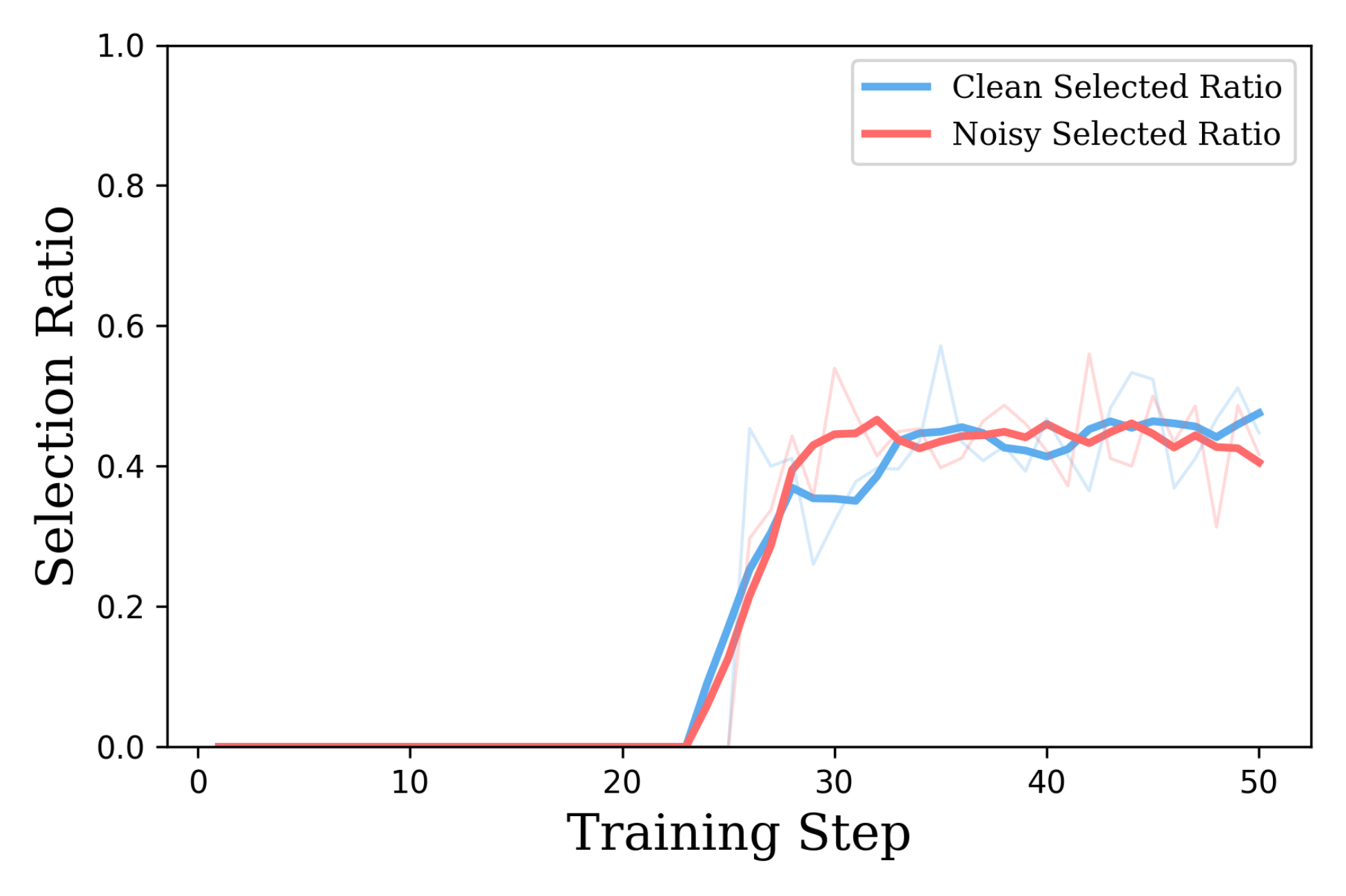}
    \caption{Selection Ratio}
  \end{subfigure}
  \hfill
  \begin{subfigure}{0.32\textwidth}
    \centering
    \includegraphics[width=\linewidth]{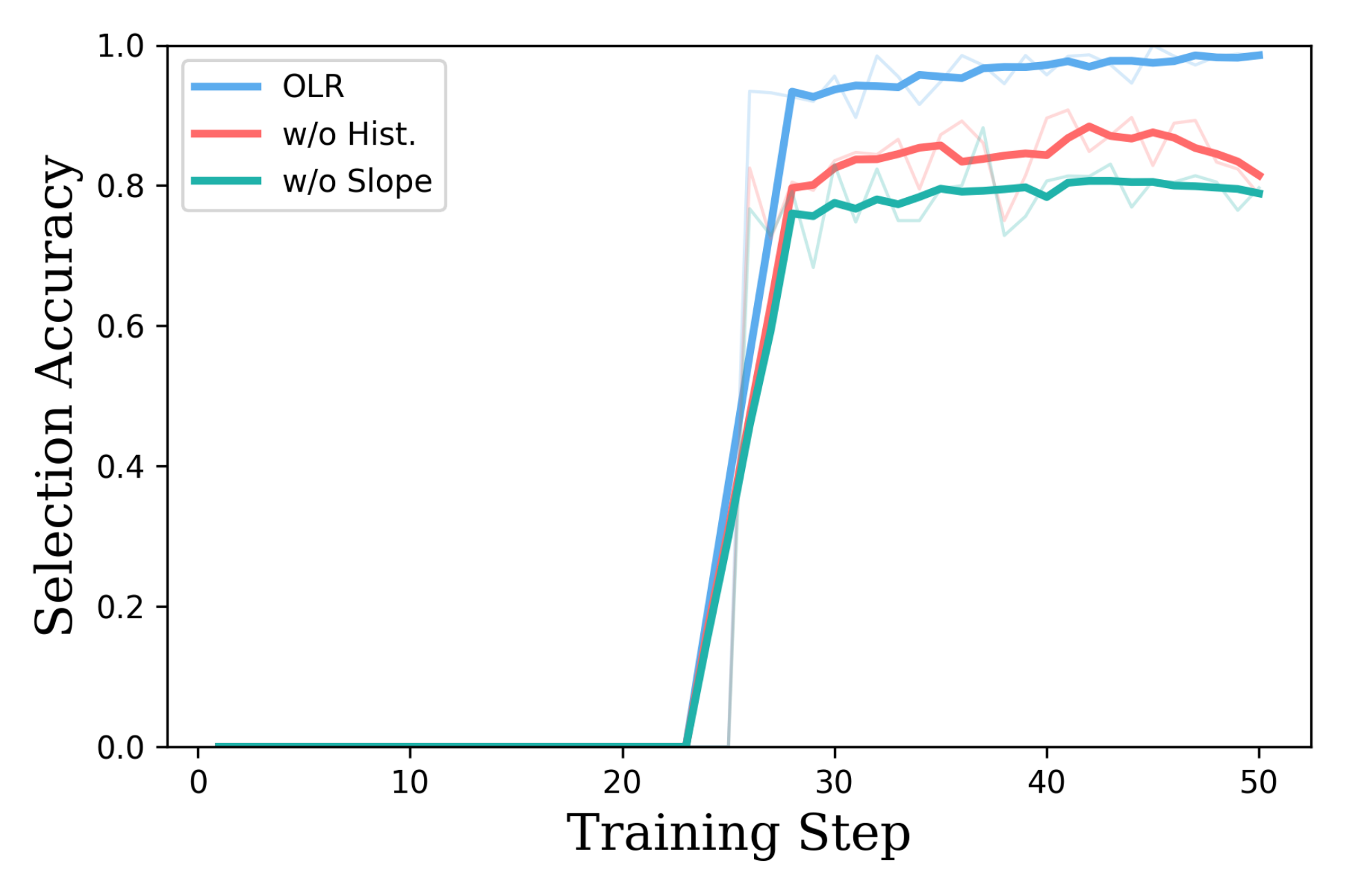}
    \caption{Ablation Accuracy}
  \end{subfigure}
  \hfill
  \begin{subfigure}{0.32\textwidth}
    \centering
    \includegraphics[width=\linewidth]{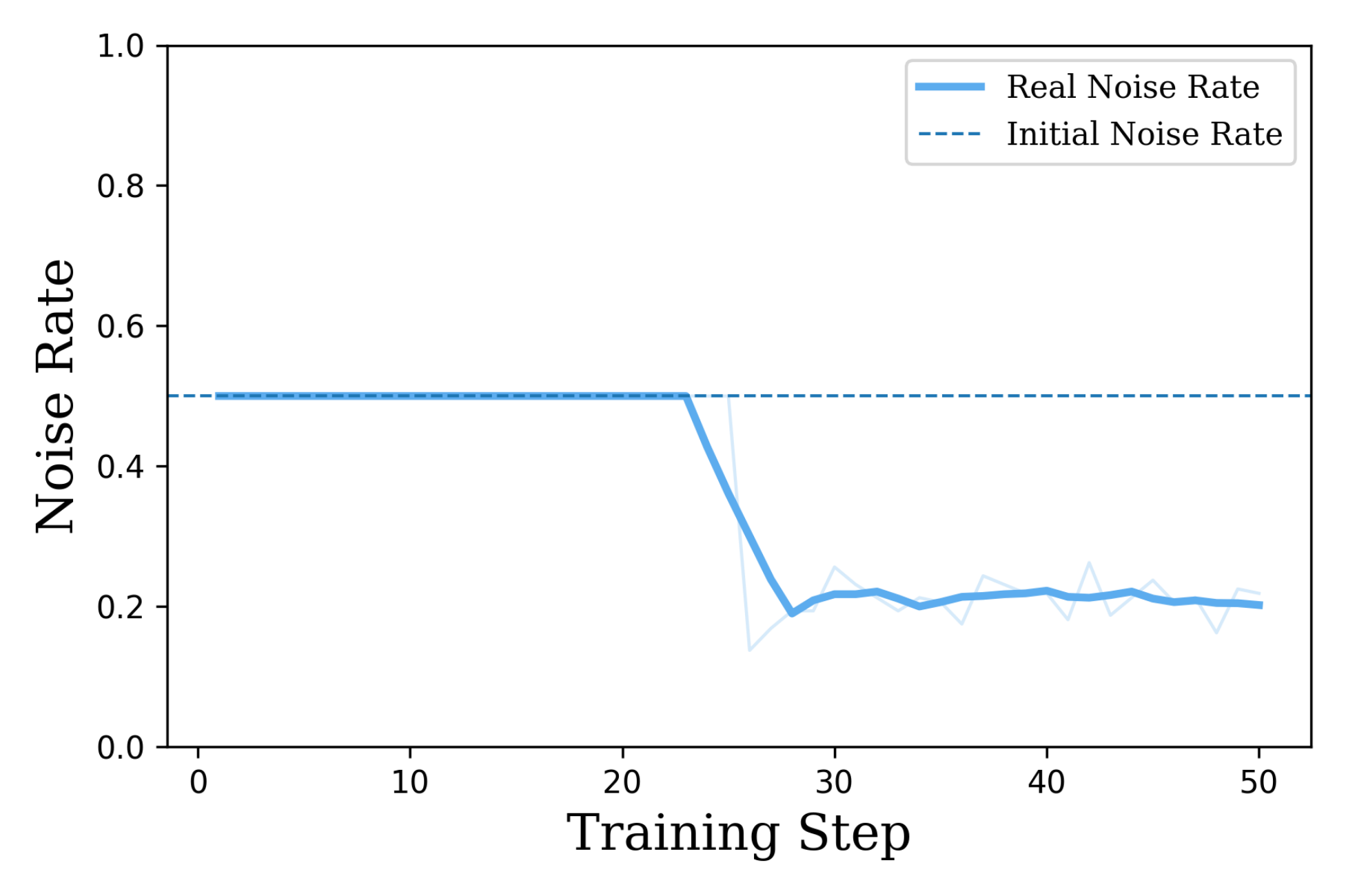}
    \caption{Real vs. Initial Noise}
  \end{subfigure}

  \caption{Training dynamics of OLR with Qwen3-4B-Base under an active noise setting (noise ratio = 0.5).}
  \label{fig:sensitive}
\end{figure*}

\paragraph{How does OLR behave during the training process?}
Figure \ref{fig:sensitive} illustrates training dynamics under 50\% active noise. As shown in \textbf{(a)}, during the early learning phase, the majority answer accuracy for both clean and noisy samples rises to over 60\%, thereby exceeding the initial 50\% correct label ratio. This early learning stage enables initial gains; however, accuracy plateaus around 60\% afterward. Once OLR is applied, accuracy further increases to over 80\%, achieving a \textbf{20\%} improvement, demonstrating that OLR builds upon the early learning stage.
In \textbf{(b)}, for samples selected by OLR, the majority answer accuracy exceeds \textbf{90\%} and approaches 100\%, whereas unselected samples remain below 70\%. Since unselected samples retain their original labels, the remaining noise consists only of original noisy labels (reduced by \textbf{30\%}) and a small fraction (below 5\%) of erroneous majority answers from selected samples.
Furthermore, \textbf{(c)} shows that the slope of the pass rate for the majority answers remains positive throughout training for both clean and noisy samples. This persistently positive slope, combined with the rising accuracy in (a), indicates that correct answers can gradually emerge with increasing probability even on noisy samples through cross-sample coupling.
\textbf{(d)} shows that OLR selects over \textbf{40\%} of samples from both clean and noisy sets, indicating it chooses sufficiently without favoring only simple samples. \textbf{(e)} demonstrates that combining both criteria achieves nearly 100\% accuracy in sample selection, while removing either reduces accuracy by about 20\%, highlighting the indispensability of both criteria.
Finally, \textbf{(f)} demonstrates that OLR reduces the proportion of complex active noise by approximately \textbf{30\%} during the refinement phase, thereby showcasing its robustness and effectiveness. Figure \ref{fig:weak_train_dynamic} in the Appendix illustrates the training dynamics in an inactive noise scenario.

\begin{figure}[!t]
\centering
\begin{subfigure}{0.48\linewidth}
    \centering
    \includegraphics[width=0.48\linewidth]{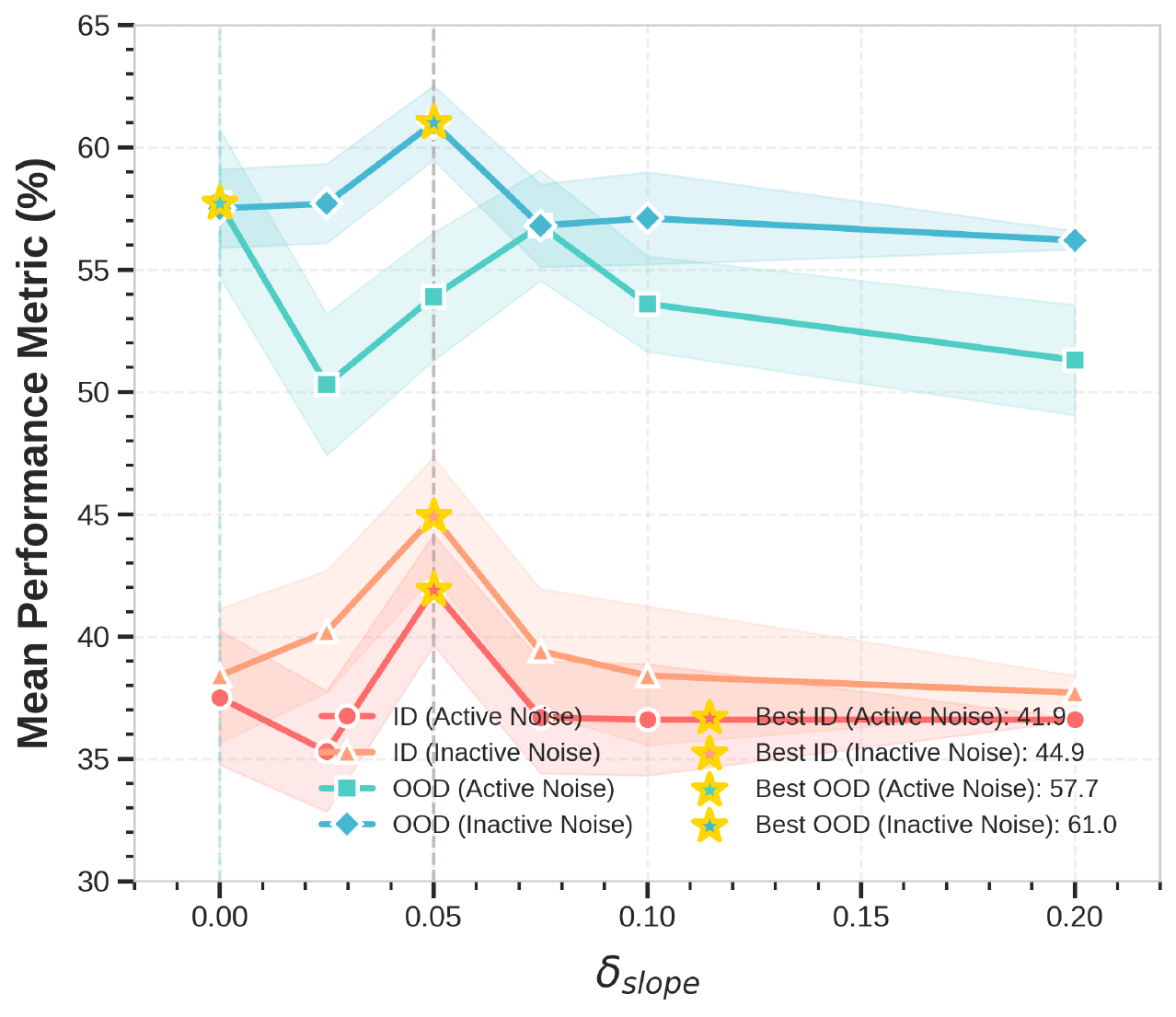}
    \includegraphics[width=0.48\linewidth]{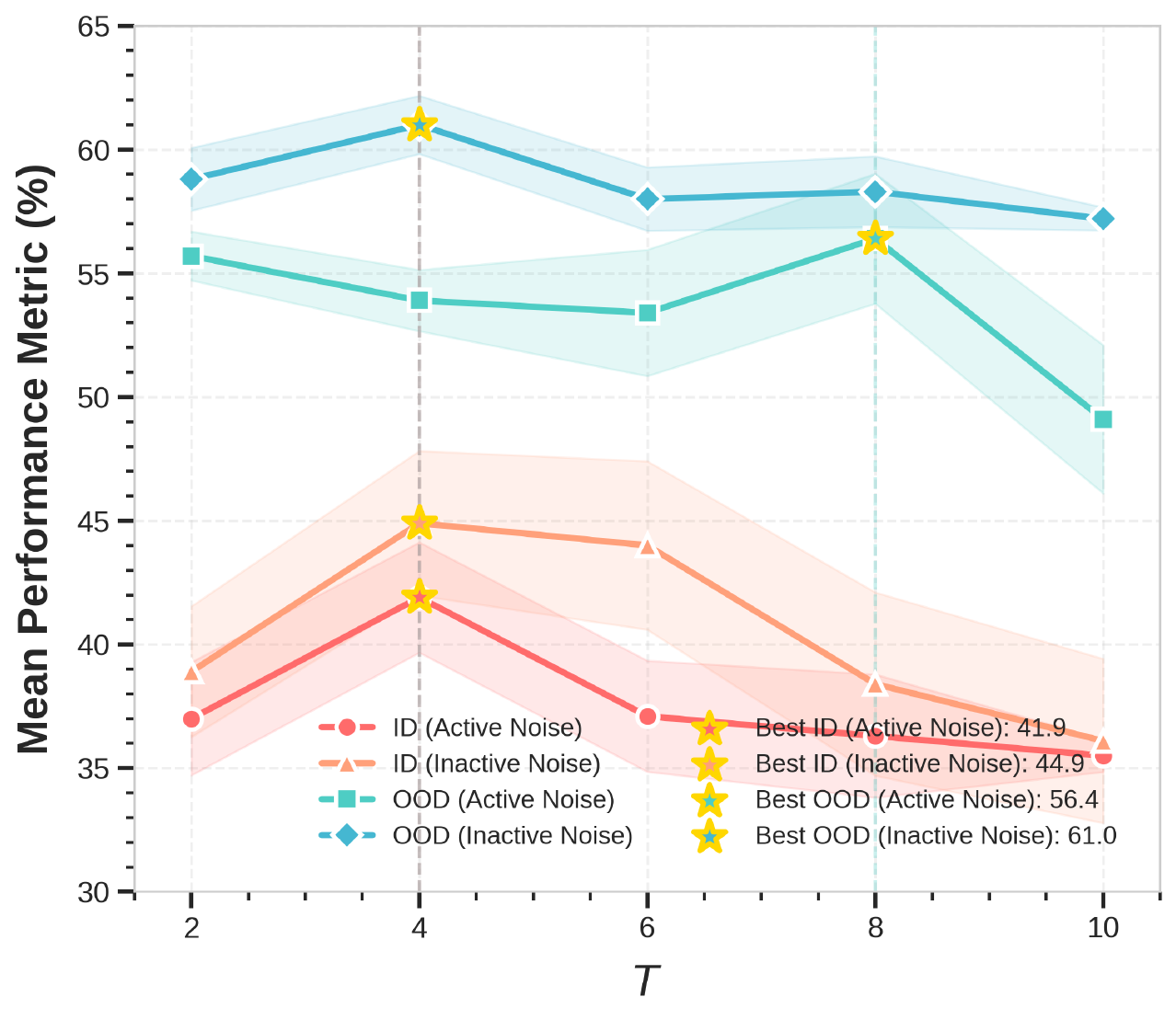}
    \caption{Sensitivity analysis}
    \label{F-sen}
\end{subfigure}
\hfill
\begin{subfigure}{0.48\linewidth}
    \centering
    \includegraphics[width=\linewidth]{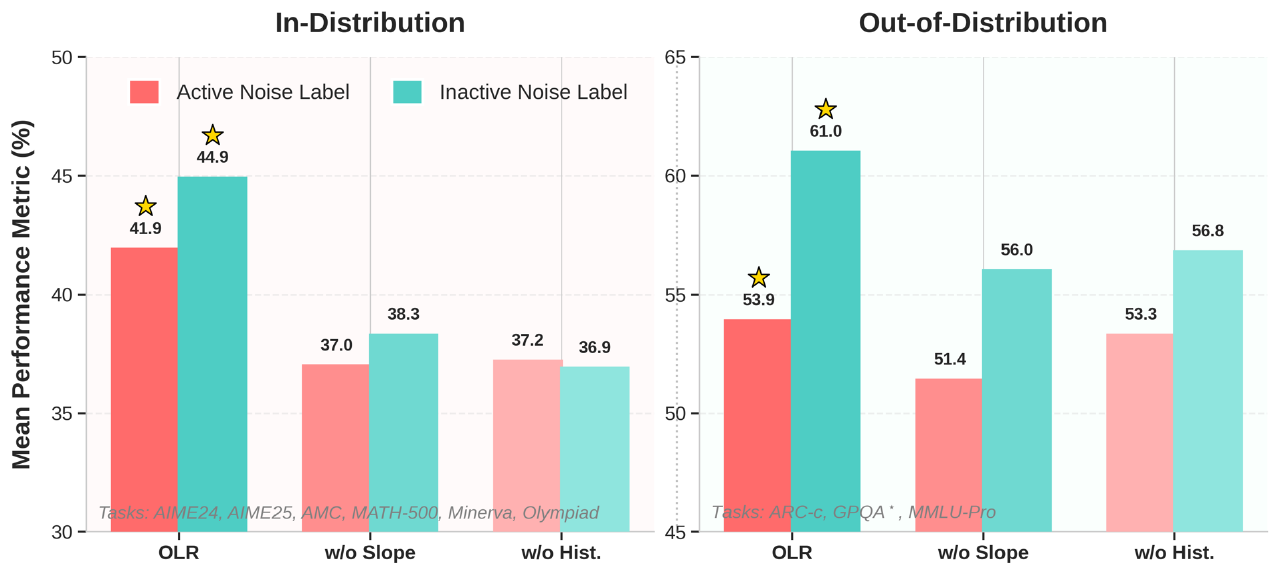}
    \caption{Ablation study}
    \label{F-aba}
\end{subfigure}
\caption{Results under a 50\% noise ratio on Qwen-3-4B-Base.}
\end{figure}



\paragraph{Does OLR's effectiveness scale with more data?} To evaluate OLR's performance with larger datasets, we conduct experiments using a training set of 4,000 samples with a 50\% noise ratio. As shown in Table \ref{tab:large_4k}, OLR continues to significantly enhance the model's reasoning performance, yielding average improvements of \textbf{4.6}\% on the ID benchmarks and \textbf{2.0}\% on the OOD benchmarks across both noise scenarios.

\paragraph{Parameter sensitivity analysis.}
We analyze OLR's sensitivity to slope threshold $\delta_{slope}$ and early learning phase duration $T$. Across all settings, OLR consistently outperforms naive noisy training (Figure~\ref{F-sen}).
Sweeping $\delta_{slope} \in \{0.0, 0.025, 0.05, 0.075, 0.1, 0.2\}$ shows optimal performance at $0.05$ for in-distribution (both noise types) and out-of-distribution under inactive noise. Too low $\delta_{slope}$ introduces erroneous majority answers; too high yields insufficient refinement samples. No clear trend emerges for out-of-distribution under active noise.
Varying $T \in \{2, 4, 6, 8, 10\}$ also yields optimal performance at $T=5$ for the same settings. A smaller $T$ causes inaccurate slope estimation; a larger $T$ prolongs noisy training. Again, no clear pattern is observed for out-of-distribution under active noise, which we aim to address in future research. Table \ref{tab:sen_slope} and \ref{tab:sen_start} in the Appendix show the specific experimental results.

\paragraph{Ablation experiments.}
We conduct ablation studies to evaluate the contribution of each criterion in OLR. As shown in Figure~\ref{F-aba}, removing either the slope or the historical consistency criterion consistently leads to a substantial performance drop, under both active and inactive noise and across ID and OOD benchmarks. Table \ref{tab:ablation} in the Appendix shows the specific experimental results. These results confirm that the two criteria are complementary and both essential to the effectiveness of OLR.


%% file: sections/conclusion.tex
\section{Conclusion}

In this paper, we take the first step toward a systematic analysis of noisy label mechanisms in RLVR. We first classify noisy labels based on rollout feasibility into active and inactive categories, establishing a taxonomy for future research. Second, we identify Early Correctness Coherence: clean and noisy samples improve similarly in early training, despite later divergence. Finally, building on this insight, we propose Online Label Refinement (OLR), the first method in RLVR that denoises labels during training. Experiments show OLR achieves significant denoising effectiveness and improves reasoning performance across various noise ratios. Beyond RLVR, our research holds promise for broader applications, including VLM and agent domains, where noisy label learning remains an underexplored direction, providing a foundation for future research across diverse paradigms.

%% file: sections/appendix.tex
\newpage

\part*{Appendix}

\section{Theoretical Proof}
\label{app:technical}

\begin{theorem}[Early Correctness Coherence in Noisy RLVR] (Formal)
\label{thm:early-main-final-formal}
Consider a dataset $\mathcal D = \mathcal D_{\rm clean} \cup \mathcal D_{\rm noise}$ with noise ratio 
$\rho = |\mathcal D_{\rm noise}| / |\mathcal D|$. 
For each prompt $x$, let   
the policy have rollout probability $p_t(y|x) = \pi_{\theta_t}(y|x)$ with $K$ rollouts per step. 
Assume small learning rate $\eta$ and early-phase dynamics (ignoring clipping effects).  
Suppose (i) initial rollout bias $p_0(y^\star) > p_0(\tilde y)$ creating a signal gap, 
(ii) positive cross-sample coupling
\begin{equation}
\Gamma(x_c,x_n) = 
\nabla_\theta \log \pi(y^\star(x_c)|x_c) \cdot \nabla_\theta \log \pi(y^\star(x_n)|x_n), 
 \quad
\mathbb E_{x_c\in \mathcal D_{\rm clean},\, x_n\in \mathcal D_{\rm noise}}[\Gamma(x_c,x_n)] \ge \gamma > 0,
\end{equation}
so that reinforcing clean samples increases correct probability on noisy samples, and 
(iii) average deterministic drift of the log-ratio
$
\Delta_s := \gamma (1-\rho) G_c - \rho G_n,
$
where $G_c$ and $G_n$ are the mean clean and noisy advantage magnitudes.
Define the log-ratio
\begin{equation}
L_t(x) = \log \frac{p_t(y^\star(x)|x)}{p_t(\tilde y(x)|x)} 
\quad \text{and} \quad
\rho_c = \frac{\gamma G_c}{\gamma G_c + G_n}.
\end{equation}
If $\rho < \rho_c$ and $K \gtrsim \log(T/\delta)$, then with probability at least $1-\delta$ simultaneously for all $t \le T$, where $T$ denotes the early convergence stage,
\begin{equation}
L_t(x) \ge L_0(x) + \eta t \Big(\Delta_s - O\!\big(\sqrt{\frac{\log(T/\delta)}{K}}\big)\Big),
\end{equation}
implying that 
$
p_t(y^\star(x)|x) \gg p_t(\tilde y(x)|x).
$
\end{theorem}


\subsection{Notation and Setup}

We consider a dataset $\mathcal D = \mathcal D_{\text{clean}}\cup\mathcal D_{\text{noise}}$ with noise ratio 
\[
\rho = \frac{|\mathcal D_{\text{noise}}|}{|\mathcal D|}.
\]

For each prompt $x\in \mathcal D$, let $y^\star(x)$ denote the correct solution and $\tilde y(x)$ denote a active noisy label. Both are assumed rollout-feasible. The policy is $\pi_\theta(y|x)$, and we define
\[
p_t(y|x) = \pi_{\theta_t}(y|x),
\]
with $K$ rollouts per step. Binary reward:
\[
r(x,y)\in\{0,1\},\quad A(x,y) = \frac{r(x,y)-\mu(x)}{\sigma(x)+\epsilon}.
\]


\subsection{Finite-Rollout Concentration Bounds}

\begin{lemma}[Advantage Concentration]
\label{lem:adv-concentration}
For any prompt $x$, let $\hat \mu(x)$ and $\hat \sigma(x)$ be the empirical mean and standard deviation computed from $K$ rollouts. Then, for any $\delta>0$, with probability at least $1-\delta$:
\begin{align}
|\hat \mu(x) - \mathbb E[\mu(x)]| &\le \sqrt{\frac{\log(2/\delta)}{2K}},\\
|\hat \sigma(x) - \sigma(x)| &\le C \sqrt{\frac{\log(2/\delta)}{K}},\\
|\hat A(x,y) - A(x,y)| &\le C' \sqrt{\frac{\log(1/\delta)}{K}},
\end{align}
where $C, C'$ are constants depending on reward bounds.
\end{lemma}

\begin{proof}
We start with the empirical mean of $K$ rollouts for a fixed prompt $x$:
\[
\hat \mu(x) = \frac{1}{K} \sum_{k=1}^{K} r(x, y^{(k)}),
\]
where each $r(x, y^{(k)}) \in [0,1]$ and $y^{(k)} \sim \pi_\theta(\cdot|x)$.

\paragraph{Step 1: Concentration of the empirical mean.}
Since the $r(x, y^{(k)})$ are independent conditioned on the policy, we can apply Hoeffding's inequality:
\[
\Pr\Big( |\hat \mu(x) - \mathbb E[\hat \mu(x)]| > \epsilon \Big) \le 2 \exp\Big(-2 K \epsilon^2 \Big),
\]
where $\mathbb E[\hat \mu(x)] = \mu(x)$ is the true expected reward for this prompt.  
Solving for $\epsilon$ with confidence level $1-\delta$ gives:
\begin{equation}
|\hat \mu(x) - \mu(x)| \le \sqrt{\frac{\log(2/\delta)}{2K}}
 \quad \text{with probability at least } 1-\delta.
\end{equation}

\paragraph{Step 2: Concentration of the empirical variance.}
The empirical variance is
\[
\hat \sigma^2(x) = \frac{1}{K} \sum_{k=1}^K (r(x, y^{(k)}) - \hat \mu(x))^2.
\]
Since $0 \le r(x, y^{(k)}) \le 1$, we have bounded deviations:
\begin{equation}
| (r(x, y^{(k)}) - \hat \mu(x))^2 - (r(x, y^{(k)}) - \mu(x))^2 |
\le | \hat \mu(x) - \mu(x) |.
\end{equation}
Thus, the empirical variance is a 1-Lipschitz function of the empirical mean.  
Combining this with Hoeffding's inequality for bounded $r$ gives
\[
|\hat \sigma(x) - \sigma(x)| \le C \sqrt{\frac{\log(2/\delta)}{K}}
\]
for some constant $C$.

\paragraph{Step 3: Concentration of normalized advantage.}
The normalized advantage is
\[
\hat A(x,y) = \frac{r(x,y) - \hat \mu(x)}{\hat \sigma(x) + \epsilon}.
\]
Consider the mapping $(r,\hat \mu, \hat \sigma) \mapsto \hat A(x,y)$.  
This mapping is Lipschitz in each argument: a deviation $\delta r$, $\delta \hat \mu$, $\delta \hat \sigma$ induces a change
\[
|\delta \hat A| \le \frac{|\delta r| + |\delta \hat \mu|}{\sigma+\epsilon} + \frac{|r-\mu|}{(\sigma+\epsilon)^2} |\delta \hat \sigma|.
\]
Since $r, \mu, \sigma$ are bounded in $[0,1]$, the total Lipschitz constant is finite.  
Therefore, combining the bounds from Steps 1--2 and applying union bound over numerator and denominator terms, we obtain
\[
|\hat A(x,y) - A(x,y)| \le C' \sqrt{\frac{\log(1/\delta)}{K}},
\]
with probability at least $1-\delta$, for some constant $C'$ depending on the reward bounds and $\epsilon$.

\end{proof}


\subsection{Log-Ratio Dynamics and Martingale Decomposition}

Define the log-ratio for a active-noise sample $x_n$:
\[
L_t(x_n) = \log \frac{p_t(y^\star(x_n)|x_n)}{p_t(\tilde y(x_n)|x_n)}.
\]

\begin{lemma}[Finite-Rollout Log-Ratio Dynamics with High-Probability Bound]
\label{lem:logratio}
Let $L_t(x_n) = \log \frac{p_t(y^\star(x_n)|x_n)}{p_t(\tilde y(x_n)|x_n)}$ be the log-ratio of clean vs active-noise solutions for sample $x_n$, and let the one-step update satisfy
\[
\Delta L_t(x_n) = \eta \big(\hat A_t(x_n,y^\star)-\hat A_t(x_n,\tilde y)\big) + \eta \Xi_t(x_n),
\]
where $\Xi_t(x_n)$ is the cross-sample coupling contribution. Define the martingale difference
\[
M_t(x_n) = \Delta L_t(x_n) - \mathbb E[\Delta L_t(x_n) \mid \mathcal F_{t-1}].
\]
Assume the finite-rollout advantage satisfies Lemma~\ref{lem:adv-concentration}. Then for rollout size $K$ sufficiently large and learning rate $\eta$ small, with probability at least $1-\delta$,
\begin{equation}
\left| L_t(x_n) - L_0(x_n) - \eta \sum_{s=0}^{t-1} \mathbb E[\Delta L_s(x_n)\mid \mathcal F_{s-1}] \right|
\le t \eta C \sqrt{\frac{\log(1/\delta)}{K}} \quad \text{for all } t\le T,
\end{equation}
where $C>0$ is a constant depending on reward bounds. In particular, $\{M_t(x_n), \mathcal F_t\}$ is a martingale difference sequence with bounded increments.
\end{lemma}

\begin{proof}
Recall the log-ratio for a active-noise sample $x_n$:
\[
L_t(x_n) = \log \frac{p_t(y^\star(x_n)|x_n)}{p_t(\tilde y(x_n)|x_n)}.
\]

\paragraph{Step 1: One-step stochastic update.}
Under GRPO, the parameter update is
\[
\theta_{t+1} = \theta_t + \eta \sum_{y\in \mathcal{Y}_t(x_n)} \hat A_t(x_n,y) \nabla_\theta \log \pi_\theta(y|x_n),
\]
where $\hat A_t$ is computed from $K$ rollouts.  
Using first-order approximation for small $\eta$, the log-ratio change is
\begin{equation}
\Delta L_t(x_n) := L_{t+1}(x_n) - L_t(x_n) 
\approx 
\eta \Big(\hat A_t(x_n, y^\star) - \hat A_t(x_n, \tilde y)\Big) + \eta \Xi_t(x_n),
\end{equation}
where $\Xi_t$ is the cross-sample coupling term arising from shared parameters across prompts.

\paragraph{Step 2: Conditional expectation and martingale decomposition.}
Define the filtration $\mathcal F_{t-1}$ as the sigma-algebra generated by all previous rollouts and policy parameters up to step $t-1$.  
Conditional on $\mathcal F_{t-1}$, the stochastic gradient $\hat A_t(x_n, y)$ is unbiased:
\[
\mathbb E[\hat A_t(x_n, y) \mid \mathcal F_{t-1}] = A_t(x_n, y),
\]
since the rollouts $y^{(k)} \sim \pi_{\theta_t}(\cdot|x_n)$ are independent given $\theta_t$.  

Define the martingale difference sequence:
\[
M_t(x_n) := \Delta L_t(x_n) - \mathbb E[\Delta L_t(x_n) \mid \mathcal F_{t-1}].
\]

By construction:

1. $\mathbb E[M_t(x_n) \mid \mathcal F_{t-1}] = 0$, i.e., $\{M_t(x_n), \mathcal F_t\}$ is a martingale difference sequence.
2. The magnitude is bounded due to bounded rewards and Lemma~\ref{lem:adv-concentration}:
\begin{equation}
|M_t(x_n)| = |\Delta L_t(x_n) - \mathbb E[\Delta L_t(x_n)\mid \mathcal F_{t-1}]| 
\le \eta C \sqrt{\frac{\log(1/\delta)}{K}}
\quad \text{with probability } 1-\delta,
\end{equation}
where the constant $C$ comes from the Lipschitz dependence of the log-ratio on normalized advantage.

\paragraph{Step 3: Martingale decomposition.}
With this definition, we can write
\[
L_{t+1}(x_n) = L_0(x_n) + \eta \sum_{s=0}^{t} \Delta_s(x_n) + \sum_{s=0}^{t} M_s(x_n),
\]
where $\Delta_s(x_n) = \mathbb E[\Delta L_s(x_n)\mid \mathcal F_{s-1}]$ is the deterministic drift term.  
This decomposition separates the deterministic drift from the stochastic fluctuation (martingale), allowing application of Azuma-Hoeffding inequality to bound deviations of $L_t(x_n)$ around its mean trajectory.

\paragraph{Step 4: High-probability bound.}
Applying Azuma-Hoeffding to $\sum_{s=0}^{t} M_s(x_n)$, since each $|M_s(x_n)|\le \eta C \sqrt{\log(1/\delta)/K}$, we obtain
\begin{equation}
\Pr\Big( \Big|\sum_{s=0}^{t} M_s(x_n)\Big| \ge \epsilon \Big)
\le 2 \exp\Big(-\frac{\epsilon^2}{2 t \eta^2 C^2 \log(1/\delta)/K}\Big).
\end{equation}
Setting $\epsilon = t \eta C \sqrt{\log(1/\delta)/K}$ gives probability at most $\delta$. 
\end{proof}


\subsection{Cross-Sample Coupling}

Define cross-sample coupling:
\begin{equation}
\Gamma(x_i,x_j) 
= \nabla_\theta \log \pi(y^\star(x_j)|x_j)\cdot \nabla_\theta \log \pi(y^\star(x_i)|x_i).
\end{equation}

Assume:
\[
\mathbb E_{x_c\sim \mathcal D_{\text{clean}}, x_n\sim \mathcal D_{\text{noise}}}[\Gamma(x_c,x_n)] \ge \gamma > 0.
\]

Then the deterministic drift for noise sample correct log-probability is:
\[
\Delta_t = \gamma (1-\rho) G_c - \rho G_n,
\]
where $G_c$ and $G_n$ denote average advantage magnitudes over clean and noisy samples.


\subsection{High-Probability Early Correctness Coherence}

\begin{theorem}[High-Probability Early Correctness Coherence]
\label{thm:appendix-early}
Suppose $p_0(y^\star)>p_0(\tilde y)$, $\gamma>0$, and $K \gtrsim \log(T/\delta)$. Define
\[
\rho_c = \frac{\gamma G_c}{G_n+\gamma G_c}.
\]
If $\rho<\rho_c$, then with probability at least $1-\delta$,
\begin{equation}
L_t(x_n) \ge L_0(x_n)
+ \eta t \Big(\gamma (1-\rho) G_c - \rho G_n - O(\sqrt{\frac{\log(T/\delta)}{K}})\Big)
\end{equation}
for all $t\le T$. In particular, $L_t(x_n)$ increases monotonically in early phase, implying
\[
p_t(y^\star(x_n)|x_n) \gg p_t(\tilde y(x_n)|x_n).
\]
\end{theorem}

\begin{proof}
Sum over martingale differences:
\[
L_t = L_0 + \eta \sum_{s<t} \Delta_s + \sum_{s<t} M_s.
\]
By Azuma-Hoeffding inequality, with probability $1-\delta$, $|\sum M_s|\le t \eta C \sqrt{\log(T/\delta)/K}$. The deterministic drift term is positive under $\rho<\rho_c$, yielding the stated high-probability bound.
\end{proof}


\subsection{Active Noise Collapse Condition}

\begin{theorem}[Active-Noise Collapse]
\label{thm:collapse}
If $\rho>\rho_c + O(\sqrt{\log(T/\delta)/K})$, then with high probability $L_t(x_n)\to -\infty$, i.e., active-noise solutions dominate.
\end{theorem}

\begin{proof}
Recall the log-ratio for a active-noise sample $x_n$:
\[
L_t(x_n) = \log \frac{p_t(y^\star(x_n)|x_n)}{p_t(\tilde y(x_n)|x_n)}.
\]

\paragraph{Step 1: Martingale decomposition.}
From Lemma~\ref{lem:logratio}, we have the decomposition
\[
L_t(x_n) = L_0(x_n) + \eta \sum_{s=0}^{t-1} \Delta_s(x_n) + \sum_{s=0}^{t-1} M_s(x_n),
\]
where $\Delta_s(x_n) = \mathbb E[\Delta L_s(x_n)\mid \mathcal F_{s-1}]$ is the deterministic drift and $\{M_s(x_n)\}$ is a martingale difference sequence with bounded increments
\[
|M_s(x_n)| \le \eta C \sqrt{\frac{\log(1/\delta)}{K}} \quad \text{w.p. } 1-\delta.
\]

\paragraph{Step 2: Drift under active-noise dominance.}
By definition, the deterministic drift is
\[
\Delta_s(x_n) = \mathbb E[\Delta L_s(x_n)\mid \mathcal F_{s-1}] 
= \gamma (1-\rho) G_c - \rho G_n,
\]
where $G_c$ and $G_n$ are the average clean and noisy advantage magnitudes, and $\gamma>0$ is the cross-sample coupling coefficient.  

Under the condition
\[
\rho > \rho_c + O\Big(\sqrt{\frac{\log(T/\delta)}{K}}\Big), \quad 
\text{with } \rho_c = \frac{\gamma G_c}{G_n + \gamma G_c},
\]
we have
\[
\Delta_s(x_n) \le - \epsilon < 0
\]
for some positive $\epsilon$ depending on $(\rho - \rho_c)$ and finite-rollout corrections.  
This ensures that the deterministic drift is negative at every step.

\paragraph{Step 3: Finite-rollout martingale corrections.}
The martingale sum $\sum_{s=0}^{t-1} M_s(x_n)$ introduces stochastic fluctuations.  
By Azuma-Hoeffding inequality, for all $t \le T$:
\[
\Pr\Bigg( \Big| \sum_{s=0}^{t-1} M_s(x_n) \Big| \ge t \eta C \sqrt{\frac{\log(1/\delta)}{K}} \Bigg) \le \delta.
\]

For sufficiently large $K$, the martingale fluctuation is smaller in magnitude than the deterministic negative drift:
\[
t \eta C \sqrt{\frac{\log(1/\delta)}{K}} < t \eta \epsilon,
\]
so the stochastic perturbation does not reverse the sign of the drift.

\paragraph{Step 4: Log-ratio monotonic decrease.}
Combining Steps 2 and 3, we have with probability at least $1-\delta$:
\begin{equation}
L_t(x_n) = L_0(x_n) + \underbrace{\sum_{s=0}^{t-1} \Delta_s(x_n)}_{\text{negative drift}} 
+ \underbrace{\sum_{s=0}^{t-1} M_s(x_n)}_{\text{bounded fluctuation}} 
\le L_0(x_n) - t \eta (\epsilon - C \sqrt{\frac{\log(1/\delta)}{K}}).
\end{equation}

As $t \to \infty$, the negative drift dominates and the log-ratio decreases monotonically, hence
\[
L_t(x_n) \to -\infty,
\]
implying that
\[
p_t(y^\star(x_n)|x_n) \ll p_t(\tilde y(x_n)|x_n),
\]
i.e., active-noise solutions dominate the policy with high probability.

\end{proof}


\subsection{KL-Regularized Dynamics}

With KL regularization term $\beta D_{\text{KL}}(\pi_\theta\|\pi_{\text{ref}})$, the replicator drift becomes
\[
\Delta_t^{KL} = \gamma (1-\rho) G_c - \rho G_n - \beta \Delta_{\text{ref}},
\]
where 
\[
\Delta_{\text{ref}} = \log \frac{p_{\text{ref}}(y^\star)}{p_{\text{ref}}(\tilde y)}.
\]

\begin{theorem}[KL-Regularized Phase Boundary]
\label{thm:kl-boundary}
With KL regularization, the critical noise ratio shifts to
\[
\rho_c^{KL} = \frac{\gamma G_c - \beta \Delta_{\text{ref}}}{G_n + \gamma G_c}.
\]
The high-probability Early Correctness Coherence and collapse conditions are modified accordingly.
\end{theorem}


\begin{proof}
We consider the same log-ratio for a active-noise sample $x_n$:
\[
L_t(x_n) = \log \frac{p_t(y^\star(x_n)|x_n)}{p_t(\tilde y(x_n)|x_n)}.
\]

\paragraph{Step 1: One-step update with KL regularization.}
Under GRPO with KL regularization, the surrogate gradient includes a term:
\begin{equation}
-\beta \nabla_\theta D_{\text{KL}}(\pi_\theta\|\pi_{\text{ref}}) 
= -\beta \mathbb E_{y\sim \pi_\theta}[\nabla_\theta \log \pi_\theta(y|x_n) \log \frac{\pi_\theta(y|x_n)}{\pi_{\text{ref}}(y|x_n)}].
\end{equation}

Using the same first-order approximation as before for small $\eta$, the deterministic drift of the log-ratio becomes
\begin{equation}
\Delta_t^{KL} := \mathbb E[\Delta L_t(x_n) \mid \mathcal F_{t-1}] 
= \gamma (1-\rho) G_c - \rho G_n - \beta \Delta_{\text{ref}},
\end{equation}
where
\[
\Delta_{\text{ref}} = \log \frac{p_{\text{ref}}(y^\star)}{p_{\text{ref}}(\tilde y)}
\]
captures the bias introduced by the reference policy.  

\paragraph{Step 2: Modified critical noise ratio.}
In the non-KL case, the phase boundary is
\[
\rho_c = \frac{\gamma G_c}{\gamma G_c + G_n}.
\]

Including the KL drift, we require the total deterministic drift to be zero at the boundary:
\[
\Delta_t^{KL} = \gamma (1-\rho_c^{KL}) G_c - \rho_c^{KL} G_n - \beta \Delta_{\text{ref}} = 0.
\]

Solving for $\rho_c^{KL}$ gives
\begin{equation}
\gamma (1-\rho_c^{KL}) G_c - \rho_c^{KL} G_n - \beta \Delta_{\text{ref}} = 0
\quad \Rightarrow \quad
\rho_c^{KL} = \frac{\gamma G_c - \beta \Delta_{\text{ref}}}{G_n + \gamma G_c}.
\end{equation}

\paragraph{Step 3: Martingale decomposition and high-probability bound.}
Define the martingale difference as in Lemma~\ref{lem:logratio}:
\[
M_t(x_n) = \Delta L_t(x_n) - \mathbb E[\Delta L_t(x_n) \mid \mathcal F_{t-1}].
\]
Its magnitude is still bounded by $\eta C \sqrt{\log(1/\delta)/K}$ due to Lemma~\ref{lem:adv-concentration}.  

Thus, the log-ratio evolves as
\[
L_t(x_n) = L_0(x_n) + \eta \sum_{s=0}^{t-1} \Delta_s^{KL} + \sum_{s=0}^{t-1} M_s(x_n),
\]
and applying Azuma-Hoeffding gives with probability $1-\delta$:
\begin{equation}
\left| L_t(x_n) - L_0(x_n) - \eta \sum_{s=0}^{t-1} \Delta_s^{KL} \right|
\le t \eta C \sqrt{\frac{\log(1/\delta)}{K}}.
\end{equation}

\paragraph{Step 4: Implications for Early Correctness Coherence  and collapse.}
- If $\rho < \rho_c^{KL}$, the deterministic drift $\Delta_t^{KL} > 0$, so $L_t(x_n)$ increases, and Early Correctness Coherence occurs.  
- If $\rho > \rho_c^{KL}$, the deterministic drift $\Delta_t^{KL} < 0$, so $L_t(x_n) \to -\infty$ with high probability, i.e., active-noise collapse occurs.  

Finite-rollout stochastic fluctuations do not change the sign of the drift if $K$ is sufficiently large, preserving the phase boundary.  

This completes the proof.
\end{proof}

\subsection{OLR Improves Label Noise Tolerance} \label{proof_thm_noise_tolerance}
\begin{proof}
We prove the theorem \ref{thm:olr-tolerance} in three steps: correctness of the selected label, reduction of the effective noise ratio, and improvement of the tolerable noise threshold.

\paragraph{Step 1: Correctness of the selected label.}
For prompt $x$ with ground-truth $y^\star(x)$, let $\pi_\theta$ generate $K$ rollouts $\mathcal{Y}_t(x)$. Denote $p^\star = \pi_\theta(y^\star(x)\mid x)$ and $p_c = \pi_\theta(c\mid x)$ for $c\neq y^\star(x)$, with probability gap $\Delta_p = p^\star - \max_{c\neq y^\star(x)} p_c >0$ under early correctness dominance (Theorem~\ref{thm:appendix-early}). Let $N_c\sim \mathrm{Binomial}(K,p_c)$, and define the majority answer $y^{\text{maj}}_t(x)=\arg\max_c N_c$.  

An incorrect selection occurs only if $N_c \ge N_{y^\star}$ for some $c\neq y^\star$, which by Chernoff and union bound gives
\[
\Pr(y^{\text{maj}}_t(x)\neq y^\star(x)) \le \epsilon, \quad \epsilon=O(\exp(-K\Delta_p^2)).
\]
Hence, whenever OLR replaces a label,
\[
\Pr(\hat y_t(x)=y^\star(x)) \ge 1-\epsilon.
\]

\begin{figure*}[!t]
  \centering
  \includegraphics[width=1.0\textwidth]{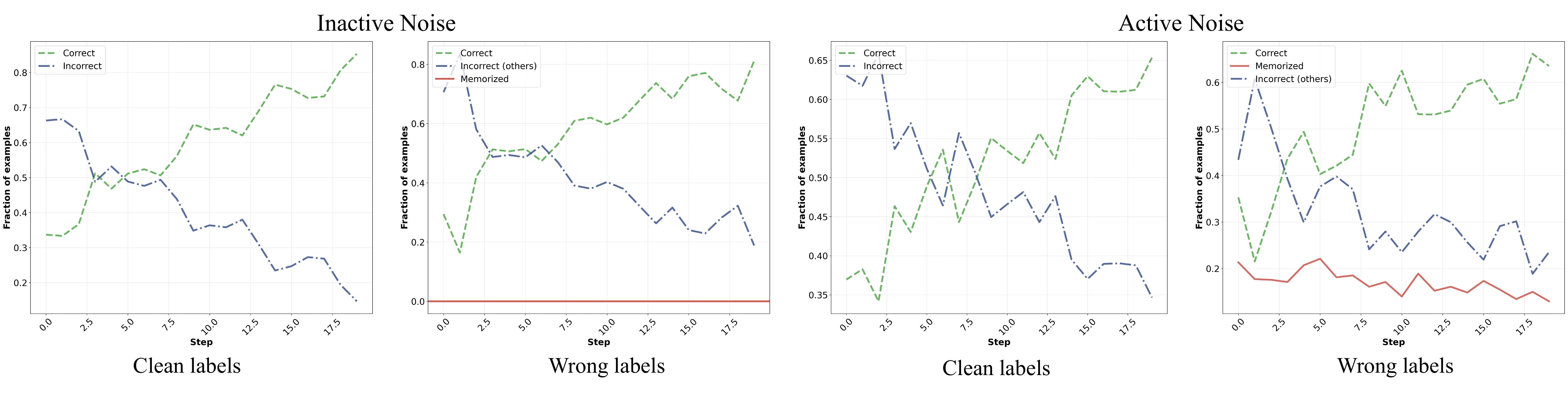}
\caption{
Results obtained with 50\% inactive or active noisy labels (800 samples) using the Qwen-3-4B-Base model \citep{yang2025qwen3}. \textbf{For each sample, we take the majority vote across multiple model rollouts as the model's prediction.}
The left and right pairs of columns illustrate the model's predictions on clean and noisy samples, respectively, under both inactive and active noise. For clean samples, we display the proportion of correct (green) and incorrect (blue) predictions. For noisy samples, we show the proportion of predictions that are correct (green), that match the noisy label (red), or that match neither the true nor the noisy label (blue). The early training phase reveals that the model learns to predict true labels even on noisy examples, indicating a preference for fitting correctly labeled samples and an increasing likelihood of producing correct answers on noisy ones over time.
}
  \label{fig:assumption_verify}
\end{figure*}

\begin{figure*}[!t]
  \centering

    \begin{subfigure}{0.32\textwidth}
    \centering
    \includegraphics[width=\linewidth]{Figure/weak_0.5_True.pdf}
    \caption{Training Accuracy vs. Steps}
  \end{subfigure}
  \hfill
    \begin{subfigure}{0.32\textwidth}
    \centering
    \includegraphics[width=\linewidth]{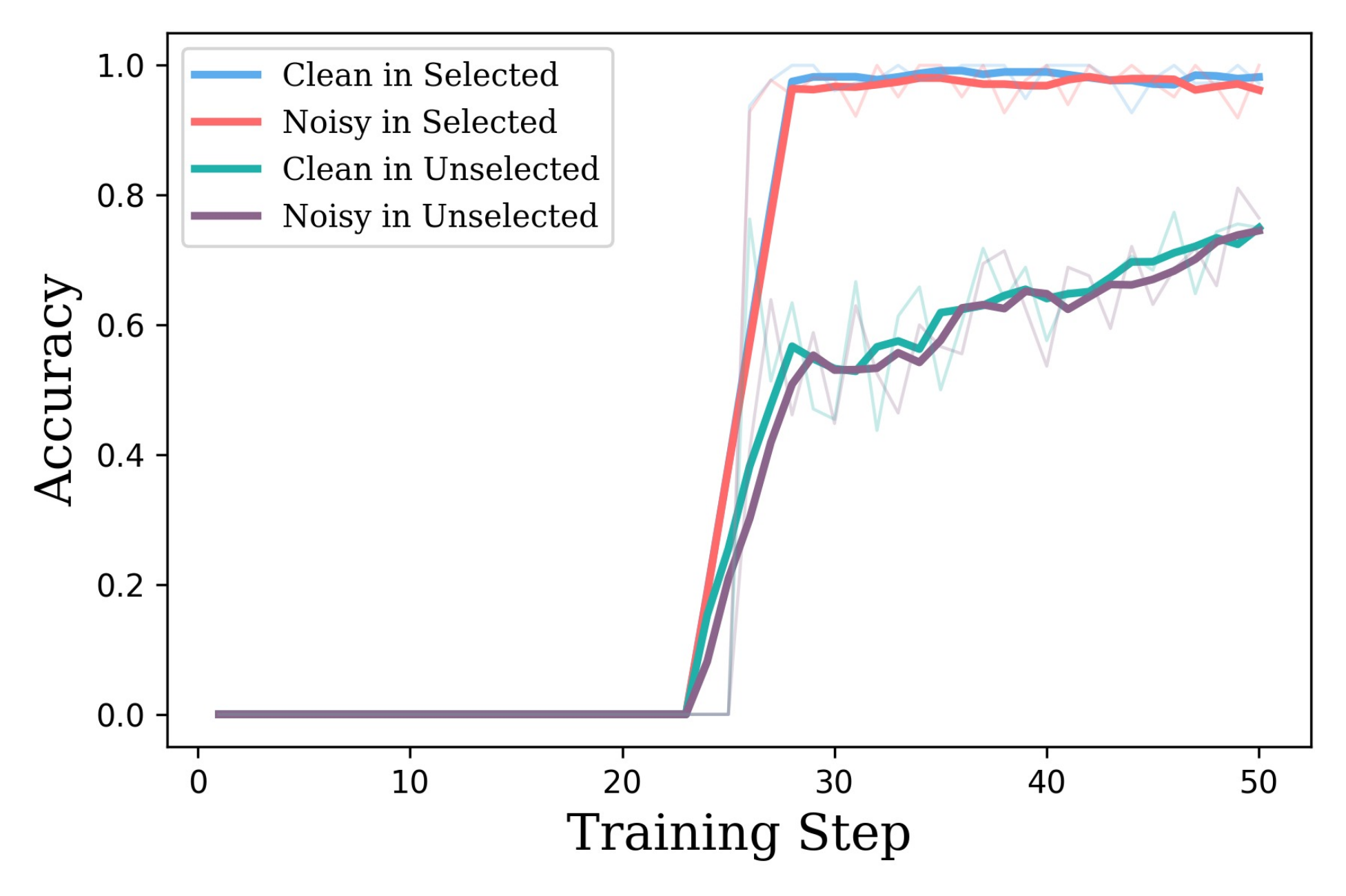}
    \caption{Selected vs. Unselected Accuracy}
  \end{subfigure}
  \hfill
  \begin{subfigure}{0.32\textwidth}
    \centering
    \includegraphics[width=\linewidth]{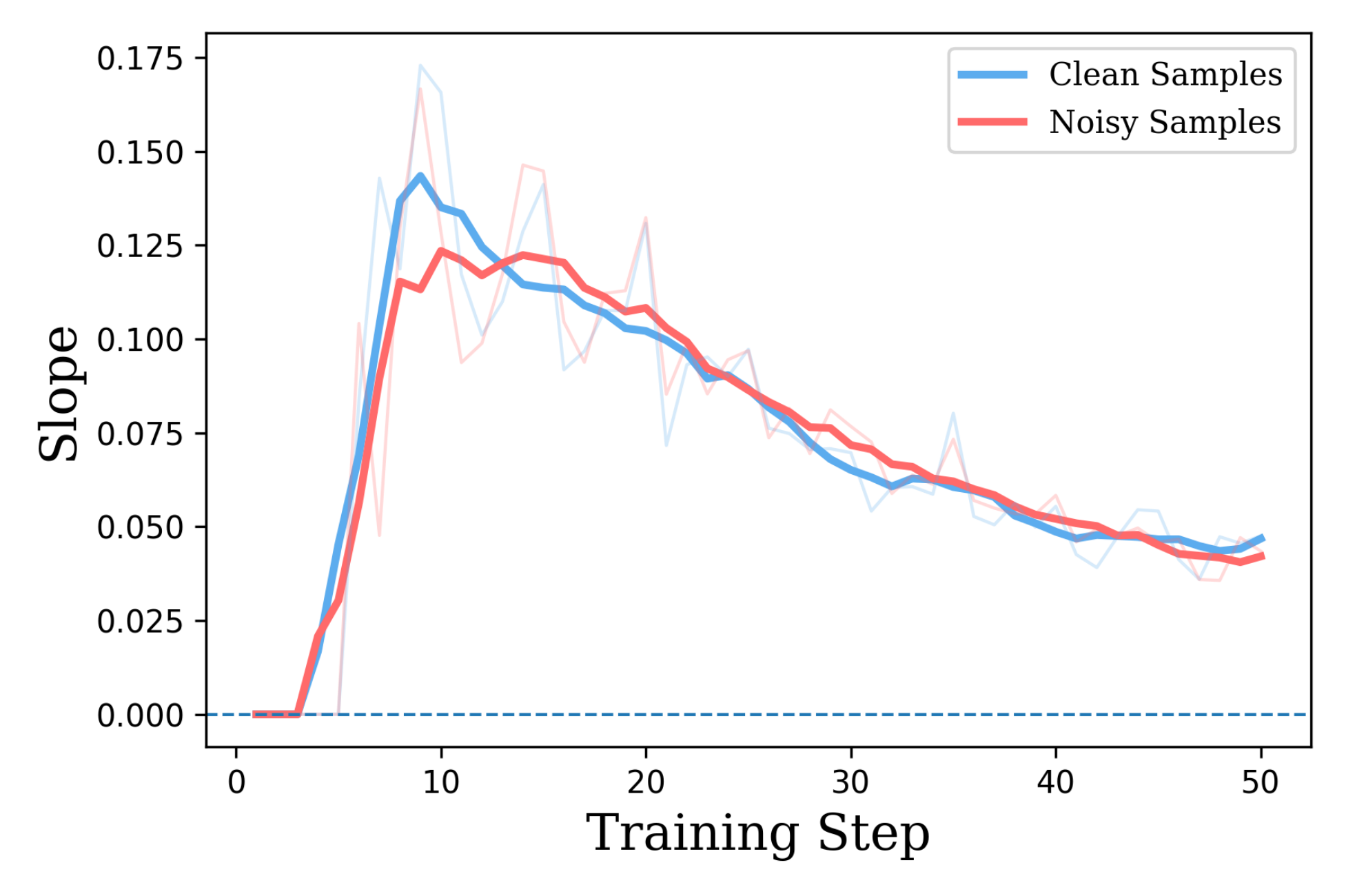}
    \caption{Passrate Slope}
  \end{subfigure}

  \vspace{2mm}

  \begin{subfigure}{0.32\textwidth}
    \centering
    \includegraphics[width=\linewidth]{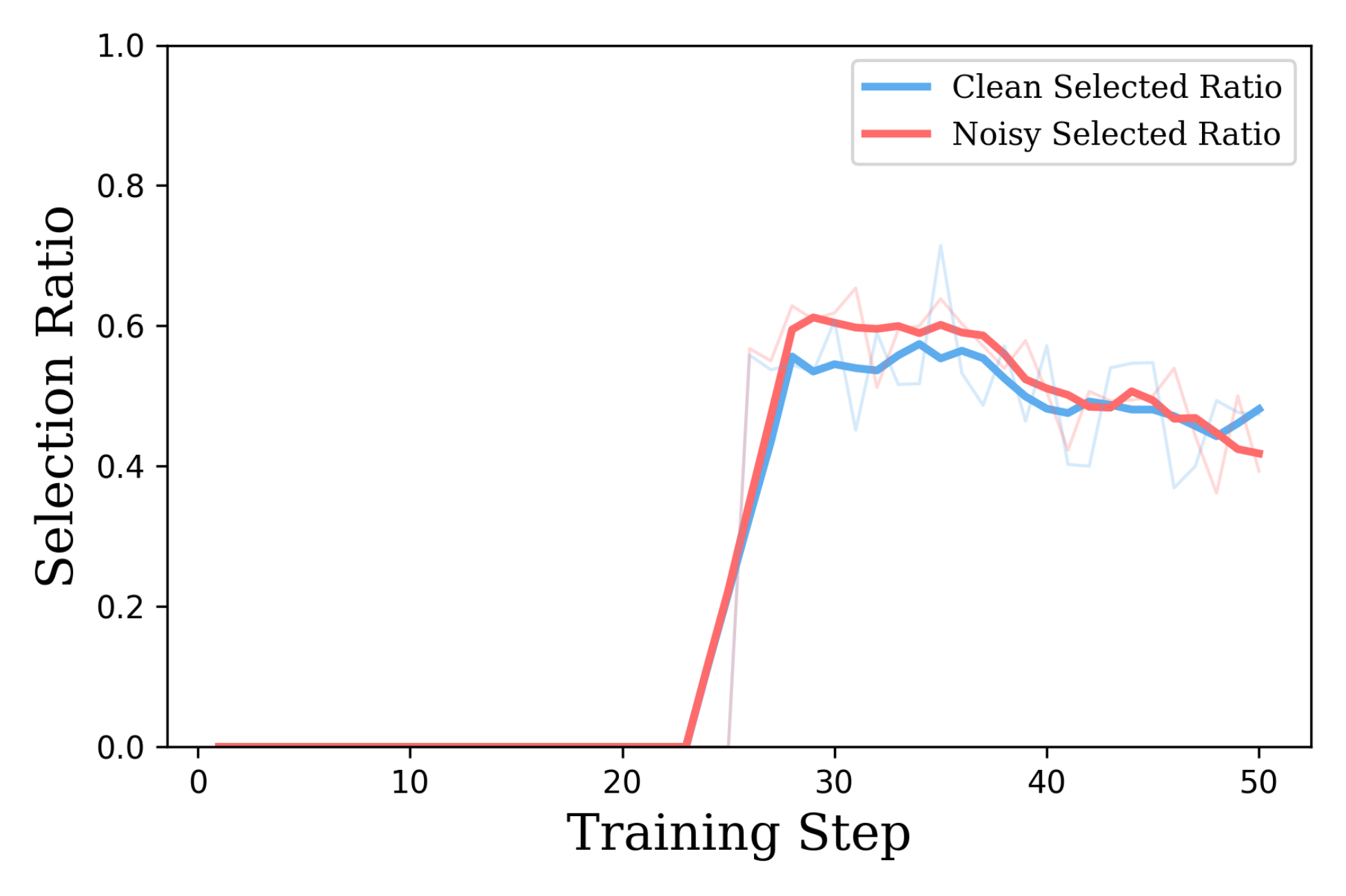}
    \caption{Selection Ratio}
  \end{subfigure}
  \hfill
  \begin{subfigure}{0.32\textwidth}
    \centering
    \includegraphics[width=\linewidth]{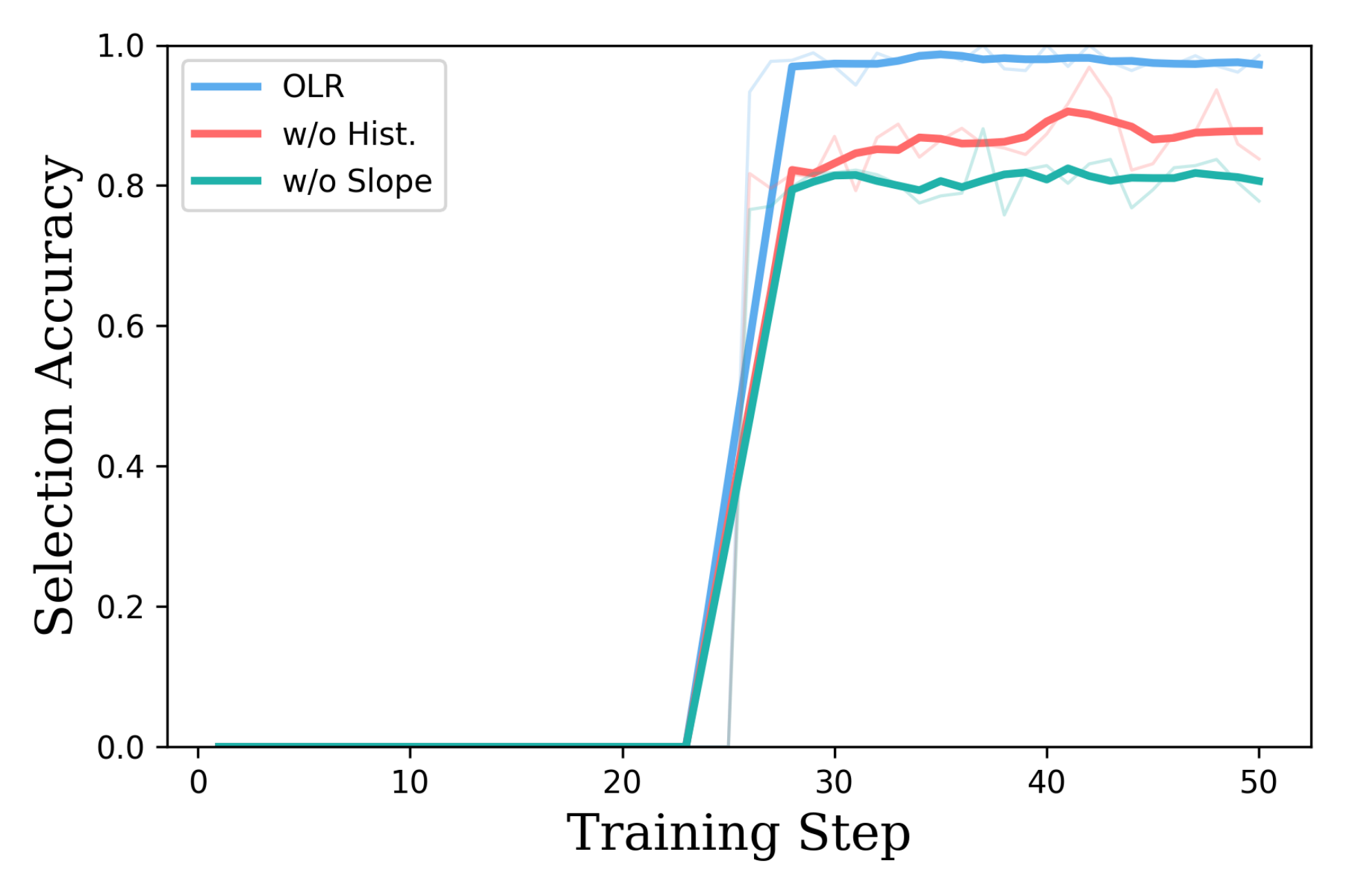}
    \caption{Ablation Accuracy}
  \end{subfigure}
  \hfill
  \begin{subfigure}{0.32\textwidth}
    \centering
    \includegraphics[width=\linewidth]{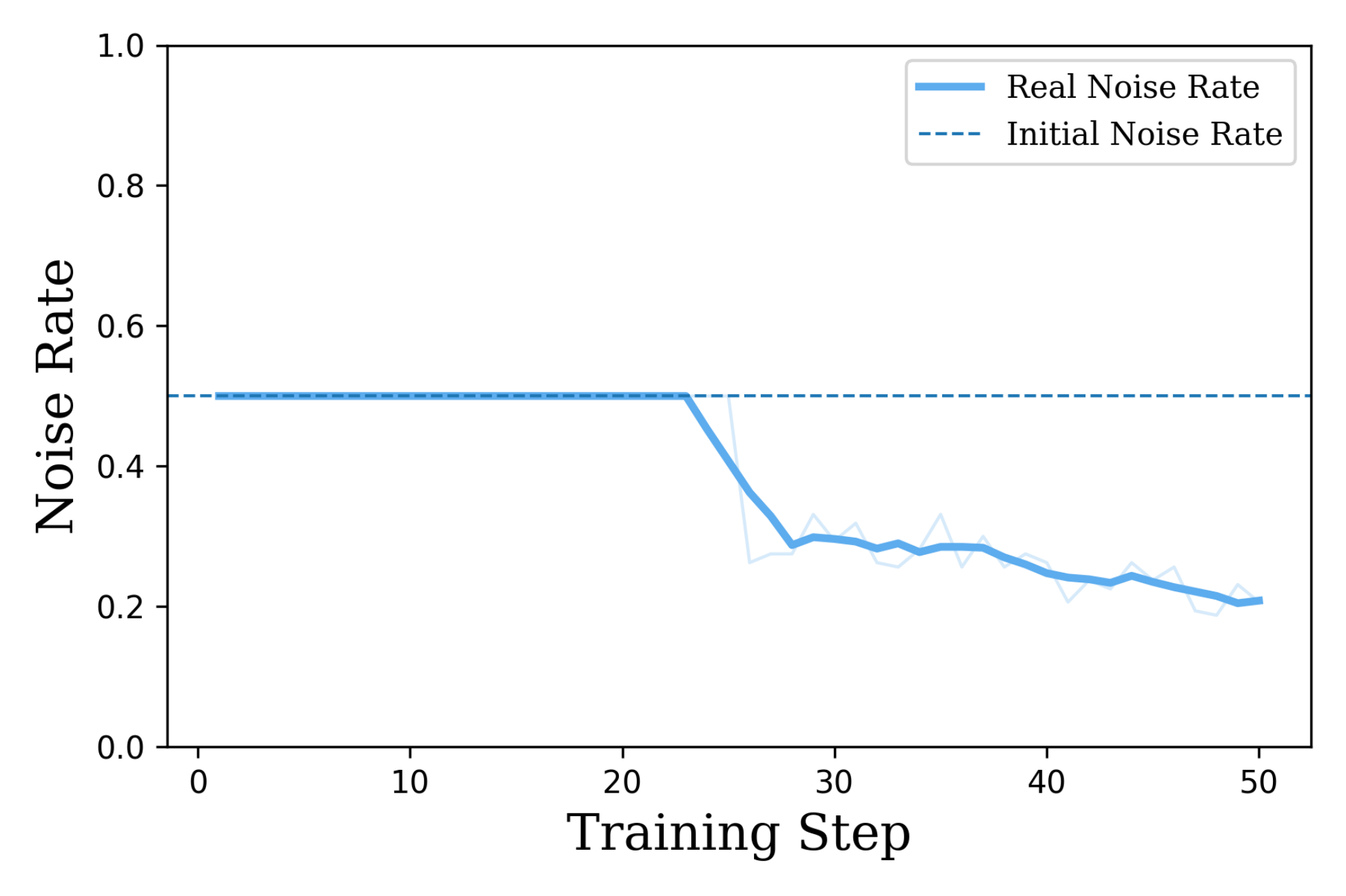}
    \caption{Real vs. Initial Noise}
  \end{subfigure}

  \caption{Training dynamic of OLR with Qwen3-4B-Base under the inactive noisy (0.5 ratio) scenario.}
  \label{fig:weak_train_dynamic}
\end{figure*}

\paragraph{Step 2: Effective noise ratio reduction.}
Let $\rho$ be the original noise ratio and $\Delta$ the probability that a noisy sample satisfies the OLR replacement criteria. Then the probability that a noisy sample remains incorrect is approximately $1-\Delta$, ignoring the exponentially small $\epsilon$, so
\[
\rho_{\text{eff}} \approx \rho(1-\Delta) < \rho.
\]

\paragraph{Step 3: Increased tolerable noise threshold.}
Let $\rho_c$ be the critical noise ratio in Theorem~\ref{thm:appendix-early}. Stability now requires
\[
\rho_{\text{eff}} = \rho(1-\Delta) < \rho_c \quad \Rightarrow\quad \rho < \frac{\rho_c}{1-\Delta} =: \rho_c^{\text{OLR}}.
\]
Since $\Delta>0$, we have $\rho_c^{\text{OLR}} > \rho_c$.

\paragraph{Conclusion.}  
Combining the above shows that OLR (i) selects correct labels with high probability, (ii) reduces the effective noise ratio, and (iii) increases the maximum tolerable noise level, completing the proof.
\end{proof}

\begin{table*}[!t]
\caption{Results on Qwen3-8B-Base under 50\% noise ratio. \textbf{Bold} indicates the better.}
\label{tab:8B_800}
\centering
\small
\resizebox{\textwidth}{!}{%
\begin{tabular}{lccccccccccc}
\toprule
\multirow{2.5}{*}{\textbf{Method}} 
& \multicolumn{7}{c}{\textbf{In-Distribution}} 
& \multicolumn{4}{c}{\textbf{Out-of-Distribution}} \\
\cmidrule(lr){2-8}
\cmidrule(lr){9-12}
& AIME24 & AIME25 & AMC & MATH-500 & Minerva & Olympiad 
& \phantom{0}\phantom{0}\textbf{Avg.}\phantom{0}\phantom{0}
& ARC-c & GPQA$^{\star}$ & MMLU-Pro &  \phantom{0}\phantom{0}\textbf{Avg.}\phantom{0}\phantom{0}\\
\midrule
\makecell[l]{Qwen3-8B-Base} 
&10.0	&7.5	&38.0	&61.2	&27.2	&33.3	&{29.5}		&28.8	&13.6	&46.9	&{29.8}\\
\midrule
\multicolumn{11}{l}{\emph{\quad \textbf{Active Noisy Label}}} \\

w/o OLR	&14.6	&12.5	&52.3	&80.2	&39.7	&43.9	&{40.5}	&59.6	&19.2	&60.3	&{46.4} \\
w/ OLR	&16.3	&15.0	&50.3	&80.4	&42.6	&45.5	&\textbf{41.7}		&75.5	&25.3	&60.1	&\textbf{53.6} \\

\midrule
\multicolumn{11}{l}{\emph{\quad \textbf{Inactive Noisy Label}}} \\
w/o OLR	&24.6	&20.4	&58.3	&85.0	&39.7	&52.7	&{46.8}		&67.2	&34.3	&62.9	&{54.8} \\
w/ OLR	&24.6	&18.8	&56.8	&86.6	&44.1	&51.6	&\textbf{47.1}		&75.5	&31.3	&62.2	&\textbf{56.3} \\
\bottomrule
\end{tabular}
}
\end{table*}

\begin{table*}[!t]
\caption{Results on Deepseek-R1-Distill-Llama-8B under 50\% noise ratio. \textbf{Bold} indicates the better.}
\label{tab:llama_8B_800}
\centering
\small
\resizebox{\textwidth}{!}{%
\begin{tabular}{lccccccccccc}
\toprule
\multirow{2.5}{*}{\textbf{Method}} 
& \multicolumn{7}{c}{\textbf{In-Distribution}} 
& \multicolumn{4}{c}{\textbf{Out-of-Distribution}} \\
\cmidrule(lr){2-8}
\cmidrule(lr){9-12}
& AIME24 & AIME25 & AMC & MATH-500 & Minerva & Olympiad 
& \phantom{0}\phantom{0}\textbf{Avg.}\phantom{0}\phantom{0}
& ARC-c & GPQA$^{\star}$ & MMLU-Pro &  \phantom{0}\phantom{0}\textbf{Avg.}\phantom{0}\phantom{0}\\
\midrule
\makecell[l]{Deepseek-R1-Distill-Llama-8B} 
&15.8	&17.9	&48.2	&72.2	&18.8	&35.7	&34.8		&24.9	&10.1	&38.7	&24.6\\
\midrule
\multicolumn{11}{l}{\emph{\quad \textbf{Active Noisy Label}}} \\

w/o OLR	&21.7	&56.2	&19.6	&72.2	&18.4	&40.4	&{38.1}		&20.6	&13.1	&41.7	&{25.1} \\
w/ OLR	&25.8	&17.5	&60.5	&73.8	&18.4	&39.3	&\textbf{39.2}		&27.2	&14.6	&41.5	&\textbf{27.8} \\

\midrule
\multicolumn{11}{l}{\emph{\quad \textbf{Inactive Noisy Label}}} \\
w/o OLR	&28.9	&22.9	&65.8	&80.0	&22.8	&45.2	&{44.2}		&31.0	&18.7	&45.0	&{31.6} \\
w/ OLR	&34.2	&20.4	&65.1	&80.0	&24.3	&45.6	&\textbf{44.9}	&33.8	&19.2	&45.0	&\textbf{32.7} \\
\bottomrule
\end{tabular}
}
\end{table*}

\section{Experiment Details}

\subsection{Detailed Setup}
All experiments were conducted on 8 × NVIDIA A100 (80G) GPUs. The validation and test sets followed those used in prior work \citep{luffyyan2025learning}. We set the hyperparameters to $\delta_{\text{slope}}=0.05$ and $T=5$, with a maximum response length of 4096 and 8 rollouts. During training, we used a learning rate of 1e-6 and a KL loss weight of 1e-3. The temperature coefficient was set to 1.0 during training, and to 0.6 for both validation and testing. Our implementation is built upon veRL\footnote{https://github.com/volcengine/verl}, which uses vLLM\footnote{https://github.com/vllm-project/vllm} as the rollout generator. We thank the authors for making these open-source repositories available.

\subsection{Baseline Description} \label{subsec:baseline_details}
We compare two categories of baseline methods.
The first category consists of unsupervised approaches that do not rely on ground-truth labels:
    \begin{itemize}
        \item \textbf{TTRL} \citep{zuo2025ttrl}: uses majority-voted outputs as pseudo-labels;
        \item \textbf{Co-Reward} \citep{zhang2025co}: leverages contrastive consistency to alleviate collapse in unsupervised RLVR training;
        \item \textbf{Self-Certainty} \citep{zhao2025learning}: promotes high-confidence predictions by maximizing KL divergence;
        \item \textbf{Token-Level Entropy} \citep{agarwal2025unreasonable}: enforces consistency by minimizing token-level entropy;
        \item \textbf{Sentence-Level Entropy} \citep{agarwal2025unreasonable}: encourages coherent predictions by maximizing sentence-level likelihood.
    \end{itemize}
    
The second category includes noise-robust learning methods and transfer-friendly regularization techniques:
    \begin{itemize}
        \item \textbf{Confidence Penalty} \citep{pereyra2017regularizing}: discourages overconfident predictions to reduce overfitting to noisy labels (implemented at the token level in our setting);
        \item \textbf{Label Smoothing} \citep{lukasik2020does}: replaces one-hot targets with smoothed distributions to mitigate overconfidence and noise sensitivity (also applied at the token level);
        \item \textbf{Small-loss Selection} \citep{gui2021towards}: treats low-loss samples as clean and prioritizes them for training;
        \item \textbf{Random Selection}: randomly samples training instances as a control variant.
    \end{itemize}

\section{More Experiments}\label{apdix:more_experiment}
\paragraph{OLR performance on models of different sizes and types.}
We run OLR on Qwen3-8B-Base. The experiments in Table \ref{tab:8B_800} show that OLR also delivers performance gains, particularly on OOD benchmarks, where it achieves an average improvement of \textbf{4.4\%} over models trained without OLR across both noise scenarios. At the same time, the results shown in Table \ref{tab:llama_8B_800} demonstrate that OLR improves performance on Deepseek-R1-Distill-Llama-8B \citep{grattafiori2024llama} across both ID and OOD benchmarks. This clearly demonstrates the versatility of OLR.

\begin{table*}[!t]
\caption{Results on Qwen3-4B-Base under 50\% noise ratios with different slope threshold $\delta_{\text{slope}}$. \textbf{Bold} indicates the best.}
\label{tab:sen_slope}
\centering
\small
\resizebox{\textwidth}{!}{%
\begin{tabular}{lccccccccccc}
\toprule
\multirow{2.5}{*}{\textbf{Method}} 
& \multicolumn{7}{c}{\textbf{In-Distribution}} 
& \multicolumn{4}{c}{\textbf{Out-of-Distribution}} \\
\cmidrule(lr){2-8}
\cmidrule(lr){9-12}
& AIME24 & AIME25 & AMC & MATH-500 & Minerva & Olympiad 
& \phantom{0}\phantom{0}\textbf{Avg.}\phantom{0}\phantom{0}
& ARC-c & GPQA$^{\star}$ & MMLU-Pro &  \phantom{0}\phantom{0}\textbf{Avg.}\phantom{0}\phantom{0}\\
\midrule

\makecell[l]{Qwen3-4B-Base} 
& 9.6 & 4.2 & 34.2 & 52.6 & 19.5 & 28.4 & 24.7 
& 35.8 & 14.1 & 33.3 & 27.7\\

\midrule
\multicolumn{11}{l}{\emph{\quad \textbf{Active Noisy Label}}} \\

GRPO	&10.4	&8.8	&45.8	&74.6	&33.1	&40.4	&35.5		&73.9	&24.7	&48.6	&49.1 \\
OLR w/ $\delta_{\text{slope}} = 0.0$	&13.3	&9.2	&46.8	&76.4	&38.2	&40.7	&{37.4}		&85.4	&30.8	&56.8	&\textbf{57.7} \\
OLR w/ $\delta_{\text{slope}} = 0.025$ 	&10.0	&10.0	&43.1	&76.0	&34.2	&38.8	&{35.4}		&73.5	&27.3	&50.2	&{50.3} \\
OLR w/ $\delta_{\text{slope}} = 0.050$	&20.4	&15.4	&49.7	&81.4	&36.4	&48.1	&\textbf{41.9}		&78.5	&28.3	&54.8	&{53.9} \\
OLR w/ $\delta_{\text{slope}} = 0.075$	&12.5	&11.3	&45.8	&75.4	&36.2	&39.4	&{36.8}		&82.6	&31.3	&56.4	&{56.8} \\
OLR w/ $\delta_{\text{slope}} = 0.1$	&10.8	&8.8	&44.3	&76.4	&37.5	&41.8	&{36.6}		&80.0	&26.8	&54.1	&{53.6} \\
OLR w/ $\delta_{\text{slope}} = 0.2$	&9.6	&8.8	&46.1	&76.2	&40.1	&38.8	&{36.6}		&73.0	&26.8	&54.2	&{51.3}\\

\midrule
\multicolumn{11}{l}{\emph{\quad \textbf{Inactive Noisy Label}}} \\
GRPO	&12.5	&5.0	&45.3	&76.0	&34.9	&42.7	&36.1		&86.7	&28.8	&56.0	&57.2 \\
OLR w/ $\delta_{\text{slope}} = 0.0$	&12.9	&9.2	&48.0	&78.2	&40.1	&41.8	&{38.4}		&88.0	&28.3	&56.1	&{57.5} \\
OLR w/ $\delta_{\text{slope}} = 0.025$ &12.9	&15.0	&47.1	&79.8	&43.4	&43.1	&{40.2}		&85.6	&30.8	&56.6	&{57.7} \\
OLR w/ $\delta_{\text{slope}} = 0.050$ 	&21.7	&18.3	&53.9	&84.2	&42.3	&48.9	&\textbf{44.9}		&86.8	&38.4	&57.8	&\textbf{61.0} \\
OLR w/ $\delta_{\text{slope}} = 0.075$  &14.2	&9.6	&48.5	&79.2	&40.4	&44.4	&{39.4}		&82.6	&31.3	&56.4	&{56.8} \\
OLR w/ $\delta_{\text{slope}} = 0.1$ 	&10.8	&12.1	&48.3	&78.0	&39.7	&41.5	&{38.4}		&85.9	&29.3	&56.2	&{57.1} \\
OLR w/ $\delta_{\text{slope}} = 0.2$ 	&15.8	&5.4	&47.1	&77.8	&39.3	&40.7	&{37.7}		&85.8	&27.8	&54.9	&{56.2} \\

\bottomrule
\end{tabular}
}
\end{table*}

\begin{table*}[!t]
\caption{Results on Qwen3-4B-Base under 50\% noise ratios with different lengths of the early learning stage $T$. \textbf{Bold} indicates the best.}
\label{tab:sen_start}
\centering
\small
\resizebox{\textwidth}{!}{%
\begin{tabular}{lccccccccccc}
\toprule
\multirow{2.5}{*}{\textbf{Method}} 
& \multicolumn{7}{c}{\textbf{In-Distribution}} 
& \multicolumn{4}{c}{\textbf{Out-of-Distribution}} \\
\cmidrule(lr){2-8}
\cmidrule(lr){9-12}
& AIME24 & AIME25 & AMC & MATH-500 & Minerva & Olympiad 
& \phantom{0}\phantom{0}\textbf{Avg.}\phantom{0}\phantom{0}
& ARC-c & GPQA$^{\star}$ & MMLU-Pro &  \phantom{0}\phantom{0}\textbf{Avg.}\phantom{0}\phantom{0}\\
\midrule

\makecell[l]{Qwen3-4B-Base} 
& 9.6 & 4.2 & 34.2 & 52.6 & 19.5 & 28.4 & 24.7 
& 35.8 & 14.1 & 33.3 & 27.7\\

\midrule
\multicolumn{11}{l}{\emph{\quad \textbf{Active Noisy Label}}} \\

GRPO	&10.4	&8.8	&45.8	&74.6	&33.1	&40.4	&35.5		&73.9	&24.7	&48.6	&49.1 \\
OLR w/ $T=2$  &11.3	&9.6	&46.0	&75.8	&37.1	&42.5	&{37.1}		&83.4	&28.3	&55.4	&{55.7} \\
OLR w/ $T=4$	&20.4	&15.4	&49.7	&81.4	&36.4	&48.1	&\textbf{41.9}		&78.5	&28.3	&54.8	&{53.9}	\\
OLR w/ $T=6$	&9.6	&10.0	&45.9	&77.2	&39.7	&40.0	&{37.1}		&79.7	&26.8	&53.8	&{53.4}	\\
OLR w/ $T=8$	&12.1	&7.1	&43.7	&76.2	&37.5	&41.2	&{36.3}		&84.7	&31.3	&53.2	&\textbf{56.4}	\\

\midrule
\multicolumn{11}{l}{\emph{\quad \textbf{Inactive Noisy Label}}} \\
GRPO	&12.5	&5.0	&45.3	&76.0	&34.9	&42.7	&36.1		&86.7	&28.8	&56.0	&57.2 \\
OLR w/ $T=2$	&11.7	&13.3	&47.1	&79.4	&40.1	&42.1	&{39.0}		&87.7	&31.3	&57.3	&{58.8} \\
OLR w/ $T=4$	&21.7	&18.3	&53.9	&84.2	&42.3	&48.9	&\textbf{44.9}		&86.8	&38.4	&57.8	&\textbf{61.0} \\
OLR w/ $T=6$	&19.2	&17.9	&52.9	&85.0	&40.1	&49.0	&{44.0}		&82.3	&33.8	&57.9	&{58.0} \\
OLR w/ $T=8$	&12.1	&7.5	&47.1	&79.6	&41.5	&42.5	&{38.4}		&85.2	&33.8	&55.9	&{58.3} \\

\bottomrule
\end{tabular}
}
\end{table*}

\begin{table*}[!t]
\caption{Ablation results on Qwen3-4B-Base under 50\% noise ratio. \textbf{Bold} indicates the best.}
\label{tab:ablation}
\centering
\small
\resizebox{\textwidth}{!}{%
\begin{tabular}{lccccccccccc}
\toprule
\multirow{2.5}{*}{\textbf{Method}} 
& \multicolumn{7}{c}{\textbf{In-Distribution}} 
& \multicolumn{4}{c}{\textbf{Out-of-Distribution}} \\
\cmidrule(lr){2-8}
\cmidrule(lr){9-12}
& AIME24 & AIME25 & AMC & MATH-500 & Minerva & Olympiad 
& \phantom{0}\phantom{0}\textbf{Avg.}\phantom{0}\phantom{0}
& ARC-c & GPQA$^{\star}$ & MMLU-Pro &  \phantom{0}\phantom{0}\textbf{Avg.}\phantom{0}\phantom{0}\\
\midrule
\multicolumn{11}{l}{\emph{\quad \textbf{Active Noisy Label}}} \\

OLR	&20.4	&15.4	&49.7	&81.4	&36.4	&48.1	&\textbf{41.9}		&78.5	&28.3	&54.8	&\textbf{53.9} \\
- Positive Convergence Slope	&10.8	&7.9	&46.7	&79.0	&35.7	&42.1	&{37.0}		&73.9	&24.7	&55.5	&{51.4} \\
- Historical Consistency	&13.8	&10.4	&46.2	&77.2	&35.3	&40.4	&{37.2}	&78.5	&26.8	&54.5	&{53.3} \\

\midrule
\multicolumn{11}{l}{\emph{\quad \textbf{Inactive Noisy Label}}} \\
OLR	&21.7	&18.3	&53.9	&84.2	&42.3	&48.9	&\textbf{44.9}		&86.8	&38.4	&57.8	&\textbf{61.0} \\
- Positive Convergence Slope	&10.0	&10.8	&47.4	&79.6	&39.7	&42.1	&{38.3}		&86.4	&25.3	&56.2	&{56.0} \\
- Historical Consistency	&10.0	&6.7	&47.1	&78.0	&38.6	&41.0	&{36.9}		&85.1	&31.3	&54.0	&{56.8} \\

\bottomrule
\end{tabular}
}
\end{table*}

\paragraph{Detailed analysis on Early Correctness Coherence.} To verify the Early Correctness Coherence phenomenon, we conduct a detailed analysis of the model's predictions on clean and noisy samples under both noise scenarios (see Figure \ref{fig:assumption_verify}). The results substantiate the two key assumptions of Theorem \ref{thm:early-main-final}: (i) correct labeling probability is higher for clean samples than for noisy ones, reflecting the quality of pre-training; and (ii) even under noisy supervision, the probability of correct answers for noisy samples increases during early training. These empirical observations lay a solid foundation for the proposed OLR method.

\paragraph{How does OLR behave during the RLVR with inactive noisy labels?}
Figure \ref{fig:weak_train_dynamic} shows training dynamics under 50\% inactive noise. In (a), early learning lifts majority accuracy for both clean/noisy samples from 50\% to over 70\%. Accuracy then plateaus until OLR boosts it further to over 80\%, a \textbf{10\%} gain showing OLR builds on early learning. In (b), OLR-selected samples exceed \textbf{90\%} accuracy (approaching 100\%), while unselected ones stay below 80\%. (c) shows positive pass rate slopes throughout training for both sample types, indicating correct answers emerge with increasing probability on noisy samples via cross-sample coupling. (d) shows OLR selects over 40\% of both clean/noisy sets, avoiding bias toward simple samples. (e) shows both criteria together achieve near 100\% selection accuracy; removing either reduces accuracy by \textbf{20\%}, highlighting their indispensability. Finally, (f) shows OLR reduces inactive noise by \textbf{25\%} during refinement, demonstrating robustness.

\begin{table}[t]
\centering
\small
\renewcommand{\arraystretch}{1.2}
\resizebox{\columnwidth}{!}{%
\begin{tabular}{clll}
\toprule
\textbf{Phase} & \textbf{Operation} & \textbf{Compute Complexity} & \textbf{Time}\\
\midrule
RLVR Training & Policy rollout and GRPO update & $\mathcal{O}(N K L)$ & $\sim$3.8 h \\
OLR Update & Majority estimation + statistics update & $\mathcal{O}(N K)$ & negligible \\
\midrule
Total & RLVR + OLR & $\mathcal{O}(N K L)$ & $\sim$3.8 h \\
\bottomrule
\end{tabular}%
}
\caption{Training time and compute complexity of OLR with 800 samples. The additional overhead introduced by OLR is negligible compared with rollout generation.}
\label{tab:olr_time}
\end{table}

\paragraph{Time Complexity.}
Let $N$ denote the number of prompts, $K$ the number of rollouts per prompt, and $L$ the average rollout length.
The dominant cost in RLVR training is rollout generation, which requires $O(NKL)$ time per epoch due to autoregressive decoding.
OLR introduces two additional operations: majority-answer counting over $K$ rollouts and maintaining statistics for the convergence slope and historical consistency checks. Majority counting costs $O(K)$ per prompt, while the statistics can be updated incrementally in $O(1)$ time.
Therefore, the additional overhead of OLR is $O(NK)$ per epoch, which is negligible compared to the rollout cost $O(NKL)$. As shown in Table~\ref{tab:olr_time}, the empirical runtime is consistent with this analysis: training on 800 prompts takes about 3.8 hours, with negligible overhead from OLR.

\section{More Related Work}

\textbf{Noisy Label Learning} is a mature field, but its techniques rarely accommodate RLVR's open-ended generation space and on-policy dynamics.
In traditional classification scenarios, methods for mitigating noisy labels are abundant and typically fall into the following categories:

\textit{Noise Transition Matrix Estimation} \citep{xiao2015learning,chen2015webly,srivastava2014dropout,sukhbaatar2014training,patrini2017making,goldberger2017training,lee2019robust,yao2018deep,han2018masking,hendrycks2018using,collier2021correlated,bucarelli2023leveraging}: These approaches generally involve learning the flipping probabilities between true and noisy labels, mimicking the noise process. However, this is unsuitable for RLVR scenarios because noisy labels in such contexts often form an open set without inherent class boundaries, making it difficult to derive a noise transition matrix.

\textit{Loss Correction} \citep{goodfellow2014explaining,pereyra2017regularizing,zhang2017mixup,menon2020can,hendrycks2019using,tanno2019learning,jenni2018deep,zhang2021delving,wei2021open,lukasik2020does,xia2020robust,qu2021dat,cheng2021mitigating}: These methods typically employ regularized algorithms to reduce model overfitting to noisy labels. While primarily designed for classification tasks, RLVR involves generative tasks, which limits the direct applicability of most existing methods.

\textit{Small-loss-based Sample Selection} \citep{malach2017decoupling,wang2018iterative,han2018co,jiang2018mentornet,huang2019o2u,chen2019understanding,shen2019learning,yu2019does,wu2020topological,wei2020combating,wu2021ngc,song2021robust}: These approaches identify samples with low training loss as potentially clean, exploiting the phenomenon of early memorization. However, in RLVR, small-loss samples are not necessarily beneficial. For instance, when a model's rollout is either entirely correct or entirely incorrect for a sample, its loss is zero. Yet such samples offer no training value and should be avoided \citep{zhan2025exgrpo}.

\textit{Semi-supervised Learning-based Methods} \citep{berthelot2019mixmatch,song2019selfie,nguyen2019self,zhou2020robust,li2020dividemix,kim2021joint,zhang2021learning,bai2021understanding,gui2021towards,kim2021fine,ortego2021multi,yao2021jo,wang2022scalable,karim2022unicon,li2022selective,chen2023softmatch,li2023disc,feng2023ot,huang2023twin}: While abundant in traditional classification, semi-supervised methods are task-specific and scarce in RLVR, where pre-selecting clean samples from noisy data remains an unsolved challenge that prevents direct transfer.

\section{Pseudo Code}
We provide the pseudo code \ref{alg:olr}.

\begin{algorithm*}[t]
\caption{Online Label Refinement (OLR)}
\label{alg:olr}
\begin{algorithmic}[1]
\Require Dataset $D = \{(x, \tilde{y}(x))\}$, policy $\pi_\theta$, 
rollouts per prompt $K$, early learning epoch $T$, slope threshold $\delta_{\text{slope}}$
\State Initialize rollout history $ \mathcal{H}[x] \gets \emptyset $ for all $x \in D$
\For{training step $t = 1$ to MaxSteps}
    \For{each prompt $x \in D$}
        \State Generate $K$ rollouts: $\mathcal{Y}_t(x) \sim \pi_\theta(\cdot \mid x)$
        \State Compute majority answer:
        \[
            y^{\text{maj}}_t(x) = \arg\max_c |\{y \in \mathcal{Y}_t(x) : y = c\}|
        \]
        \State Compute pass rate:
        \[
            p^{\text{maj}}_t(x) = \frac{1}{K} |\{y \in \mathcal{Y}_t(x) : y = y^{\text{maj}}_t(x)\}|
        \]
        \State Update history: $\mathcal{H}[x] \gets \mathcal{H}[x] \cup \{(t, y^{\text{maj}}_t(x), p^{\text{maj}}_t(x))\}$
        \If{$t \le T$}
            \State $ \hat{y}_t(x) \gets \tilde{y}(x)$
            \State \textbf{continue}
        \EndIf
        \State Compute slope $S_t(x)$ of pass-rate trajectory via linear regression
        \State Determine historical majority:
        \[
            y^{\text{hist}}_t(x) = \arg\max_y |\{y^{\text{maj}}_{t'}(x) \in \mathcal{H}[x] : y^{\text{maj}}_{t'}(x) = y \}|
        \]
        \If{$S_t(x) > \delta_{\text{slope}}$ \textbf{and} $y^{\text{maj}}_t(x) = y^{\text{hist}}_t(x)$}
            \State $\hat{y}_t(x) \gets y^{\text{maj}}_t(x)$
        \Else
            \State $\hat{y}_t(x) \gets \tilde{y}(x)$
        \EndIf
    \EndFor
    \State Update policy $\pi_\theta$ with GRPO using $\{(x, \hat{y}_t(x))\}$
\EndFor
\end{algorithmic}
\end{algorithm*}